\documentclass{article}

\usepackage[nonatbib,preprint]{nips_2018}
\usepackage[numbers,sort&compress]{natbib}
\usepackage{algorithm}
\usepackage[noend]{algorithmic}

\usepackage[utf8]{inputenc} 
\usepackage[T1]{fontenc}    
\usepackage{url}            
\usepackage{booktabs}       
\usepackage{amsmath,amsfonts,amssymb}
\usepackage{amsfonts}       
\usepackage{nicefrac}       
\usepackage{microtype}      
\usepackage[pdftex]{graphicx}
\usepackage{subfig}
\usepackage{amsthm}
\usepackage{bbm}
\usepackage{amssymb}
\usepackage{stmaryrd}

\theoremstyle{plain}
\newtheorem{thm}{\protect\theoremname}
\providecommand{\theoremname}{Theorem}
\newtheorem{lem}{\protect\lemmaname}
\providecommand{\lemmaname}{Lemma}

\DeclareMathOperator*{\argmax}{arg\,max}

\title{Scalable Multi-Class Bayesian Support Vector Machines for Structured and Unstructured Data}

\author{
  Martin Wistuba \\
  IBM Research\\
  Dublin, Ireland\\
  \texttt{martin.wistuba@ibm.com} \\
  \And
  Ambrish Rawat \\
  IBM Research\\
  Dublin, Ireland\\
  \texttt{ambrish.rawat@ie.ibm.com} \\
}

\begin{document}

\maketitle

\begin{abstract}
We introduce a new Bayesian multi-class support vector machine by formulating a pseudo-likelihood for a multi-class hinge loss in the form of a location-scale mixture of Gaussians.
We derive a variational-inference-based training objective for gradient-based learning.
Additionally, we employ an inducing point approximation which scales inference to large data sets. 
Furthermore, we develop hybrid Bayesian neural networks that combine standard deep learning components with the proposed model to enable learning for unstructured data.
We provide empirical evidence that our model outperforms the competitor methods with respect to both training time and accuracy in classification experiments on 68 structured and two unstructured data sets.
Finally, we highlight the key capability of our model in yielding prediction uncertainty for classification by demonstrating its effectiveness in the tasks of large-scale active learning and detection of adversarial images.
\end{abstract}

\section{Introduction} 
\label{sec:introduction}

Maximum margin classifiers like support vector machines (SVMs)~\cite{Cortes1995} are arguably one of the most popular classification models. 
\citet{Cortes1995} famously introduced the linear binary SVM as a novel concept which was later extended by \citet{Boser1992} using the kernel trick~\cite{Hofmann2008} to the non-linear kernel SVM.
Numerous approaches exist to extend binary SVMs to multi-class classification tasks~\cite{Dogan2016}.
One classical approach is to combine multiple binary classifiers using the one-vs-one or one-vs-rest schemes.
Alternate approaches involve defining a multi-class hinge loss~\cite{Dogan2016}, including the seminal work by \citet{Crammer2001}, which learn a single model.

One key aspect of prediction models which is often overlooked by traditional approaches, including SVMs, is the representation and propagation of uncertainty.
In general, decision makers are not solely interested in predictions but also in the confidence about the predictions.
An action might only be taken in the case when the model in consideration is certain about its prediction.
Bayesian formalism provides a principled way to obtain these uncertainties. 
Bayesian methods handle all kinds of uncertainties in a model, be it in inference of parameters or for obtaining the predictions. 
These methods are known to be effective for online classification~\cite{Li2010}, active learning~\cite{Settles2009}, global optimization of expensive black-box functions~\cite{Jones1998}, automated machine learning~\cite{Snoek2012,Wistuba2018}, and as recently noted, even in machine learning security~\cite{Smith2018}. 

Therefore, it is natural to look for a Bayesian extension of SVM classifiers.
\citet{Polson2011} derive a pseudo-likelihood which, when being maximized, is the Bayesian equivalent to training a binary linear SVM.
\citet{Henao2014} extend this work to a non-linear version by modeling the decision function with a Gaussian process.
Furthermore, they propose the use of a sparse approximation of the Gaussian process~\cite{Snelson2005} to scale the Bayesian SVM.
More recently, the use of a variational sparse approximation~\cite{Titsias2009} has been proposed by \citet{Wenzel2017}.

The contributions in this work are threefold.
We derive a pseudo-likelihood for a multi-class hinge loss and propose a multi-class Bayesian SVM.
We provide a scalable learning scheme based on variational inference~\cite{Blei2016,Titsias2009,Hensman2015} to train the multi-class Bayesian SVM.
Additionally, we propose a hybrid Bayesian neural network which combines deep learning components such as convolutional layers with the Bayesian SVM.
This allows to jointly learn the feature extractors as well as the classifier design such that it can be applied both on structured and unstructured data.
We compare the proposed multi-class SVM on 68 structured data sets to a state-of-the-art binary Bayesian SVM with the one-vs-rest approach and the scalable variational Gaussian process~\cite{Hensman2015}.
On average, the multi-class SVM provides better prediction performance and needs up to an order of magnitude less training time in comparison to the competitor methods.
The proposed hybrid Bayesian neural network is compared on the image classification data sets MNIST~\cite{MNIST} and CIFAR-10~\cite{CIFAR10} to a standard (non-Bayesian) neural network.
We show that we achieve similar performance, however, require increased training time.
Finally, we demonstrate the effectiveness of uncertainties in experiments on active learning and adversarial detection.


\section{Related Work} 
\label{sec:related_work}

\citet{Polson2011} make a key observation and reformulate the hinge loss in the linear SVM training objective to a location-scale mixture of Gaussians.
They derive a pseudo-likelihood by introducing local latent variables for each data point and marginalize them out.
A non-linear version of this setup is considered by \citet{Henao2014} where the linear decision function is modeled as a Gaussian process.
They approximate the resulting joint posterior using Markov chain Monte Carlo (MCMC) or expectation conditional maximization (ECM).
Furthermore, they scale the inference using the fully independent training conditional approximation (FITC)~\cite{Snelson2005}.
The basic assumption behind FITC is that the function values are conditionally independent given the set of inducing points.
Then, training the Gaussian process is no longer cubically dependent on the number of training instances.
Moreover, the number of inducing points can be freely chosen.
\citet{Luts2014} extend the work of \citet{Polson2011} by applying a mean field variational approach to it.
Most recently, \citet{Wenzel2017} propose an alternate variational objective and use coordinate ascent to maximize it.
They demonstrate improved performance over a classical SVM, competitor Bayesian SVM approaches, and Gaussian process-based classifiers.

Another important related topic is Gaussian process-based classifiers~\cite{Williams1998}.
As opposed to Bayesian SVMs, these classifiers directly use a decision function with a probit or logit link function~\cite{Rasmussen2006}.
Gaussian process classifiers often perform similar to non-linear SVMs~\cite{Kuss2005} and hence, are preferred by some practitioners due to added advantages like uncertainty representation and automatic hyperparameter determination.
In this aspect the closest work to our approach is scalable variational Gaussian processes~\cite{Hensman2015}.
Like our proposed model, it tackles multi-class classification with a single model and uses variational inference with inducing point approximation to scale to large data sets.

\section{Bayesian Support Vector Machines} 
\label{sec:bayesian_support_vector_machine}

This section details the proposed multi-class Bayesian SVM.
We begin with a discussion of a Bayesian formulation of a binary SVM and follow it with the multi-class case.  

\subsection{Binary SVM} 
\label{sub:prelimaries}

\paragraph{SVM} 
\label{par:svm}

For a binary classification task, support vector machines seek to learn a decision boundary with maximum margin, i.e. the separation between the decision boundary and the instances of the two classes. 
We represent the labeled data for a binary classification task with $N$ observations and $M$-dimensional representation as $D=\{\mathbf{x}_{n},y_n\}_{n=1}^N$, where $\mathbf{x}_{n}\in\mathbb{R}^{M}$ and $y_{n}\in\left\{-1,1\right\}$ represent predictors and labels, respectively.
Training a binary SVM involves learning a decision function $f : \mathbb{R}^{M} \rightarrow \mathbb{R}$ that minimizes the regularized hinge loss,
\begin{align}
\label{eq:emp_risk}
\mathcal{L}\left(D,f,\gamma\right) = \sum_{n=1}^{N} \max\left\{1-y_nf\left(\mathbf{x}_{n}\right),0\right\} + \gamma R\left(f\right) \enspace .
\end{align}
The regularizer $R$ punishes the choice of more complex functions for $f$, and $\gamma$ is a hyperparameter that controls the impact of this regularization.
A linear SVM uses a linear decision function $f(\mathbf{x}_{n}) = \boldsymbol{\theta}^T\mathbf{x}_{n}$.
Non-linear decision functions are traditionally obtained by applying the kernel trick~\cite{Hofmann2008}.


\paragraph{Bayesian Binary SVM} 
\label{par:bayesian_binary_svm}

For the linear case, \citet{Polson2011} show that minimizing Equation \eqref{eq:emp_risk} is equivalent to estimating the mode of a pseudo-posterior (maximum a posteriori estimate)
\begin{align}
\label{eq:map_estimate}
p\left(f|D\right) \propto \exp\left(-\mathcal{L}\left(D, f, \gamma\right)\right) \propto \prod_{n=1}^N L\left(y_n|\mathbf{x}_{n},f\right)p\left(f\right),
\end{align}
derived for a particular choice of pseudo-likelihood factors $L$, defined by location-scale mixtures of Gaussians.
This is achieved by introducing local latent variables $\lambda_{n}$ such that for each instance,
\begin{align}
L\left(y_n|\mathbf{x}_{n},f\right) = \int_0^{\infty}\frac{1}{\sqrt{2\pi\lambda_n}}\exp\left(-\frac{1}{2}\frac{\left(1+\lambda_n-y_n f\left(\mathbf{x}_{n}\right)\right)^2}{\lambda_n}\right)\mathrm{d}\lambda_{n}\enspace .
\end{align}
In their formulations, \citet{Polson2011} and \citet{Henao2014} consider $\gamma$ as a model parameter and accordingly develop inference schemes.
Similar to \citet{Wenzel2017}, we treat $\gamma$ as a hyperparameter and drop it from the expressions of prior and posterior for notational convenience. 
\citet{Henao2014} extend this framework to enable learning of a non-linear decision function $f$.
Both \citet{Henao2014} and \citet{Wenzel2017} consider models where $f(x)$ is sampled from a zero-mean Gaussian process i.e. $\mathbf{f} \sim \mathcal{N}(0,K_{NN})$, where $\mathbf{f}=[f(\mathbf{x}_1),\dots,f(\mathbf{x}_n)]$ is a vector of decision function evaluations and $K_{NN}$ is the covariance function evaluated at data points.
While \citet{Henao2014} consider MCMC and ECM to learn the conditional posterior $p(\mathbf{f}|D,\boldsymbol{\lambda})$, \citet{Wenzel2017} learn an approximate posterior $q(\mathbf{f},\boldsymbol{\lambda})$ with variational inference. 

%
%
%


\subsection{Bayesian Multi-Class SVM} 
\label{sub:multi_class_svm}

A multi-class classification task involves $N$ observations with integral labels $Y = \left\{1,\dots,C\right\}$.
A classifier for this task can be modeled as a combination of a decision function $f : \mathbb{R}^{M} \rightarrow \mathbb{R}^{C}$ and a decision rule to compute the class labels,
$\hat{y}\left(\mathbf{x}_{n}\right) = \argmax_{t\in Y} f_{t}\left(\mathbf{x}_{n}\right)$.
\citet{Crammer2001} propose to minimize the following objective function for learning the decision function $f$:
\begin{align}
\mathcal{L}\left(D,f,\gamma\right) = \sum_{n=1}^{N}\max\left\{ 1+\max_{t\neq y_{n},t\in Y} f_{t}\left(\mathbf{x}_{n}\right) -f_{y_{n}}\left(\mathbf{x}_{n}\right),\ 0\right\} + \gamma R\left(f\right)\enspace ,\label{eq:cs-loss}
\end{align}
where again $\gamma$ is a hyperparameter controlling the impact of the regularizer $R$.
With the prior associated to $\gamma R\left(f\right)$, maximizing the log of Equation \eqref{eq:map_estimate} corresponds to minimizing Equation \eqref{eq:cs-loss} with respect to the parameters of $f$.
This correspondence requires the following equation to hold true for the data-dependent factors of the pseudo-likelihood,
\begin{equation}
\prod_{n=1}^{N}L\left(y_{n}\ |\ \mathbf{x}_{n},f\right)=\exp\left(-2\sum_{n=1}^{N}\max\left\{ 1+\max_{t\neq y_{n},t\in Y} f_{t}\left(\mathbf{x}_{n}\right) - f_{y_{n}}\left(\mathbf{x}_{n}\right),\ 0\right\} \right)\enspace.
\end{equation}
Analogously to \citet{Polson2011}, we show that $L\left(y_n\ |\ \mathbf{x}_{n}, f\right)$ admits a location-scale mixture of Gaussians by introducing local latent variables $\boldsymbol{\lambda} = [\lambda_1,\dots,\lambda_n]$.
This requires the lemma established by \citet{Andrews1974}.
\begin{lem}
\label{lem:lemma}For any $a,b>0$,
\begin{equation}
\int_{0}^{\infty}\frac{a}{\sqrt{2\pi\lambda}}e^{-\frac{1}{2}\left(a^{2}\lambda+b^{2}\lambda^{-1}\right)}\mathrm{d}\lambda=e^{-\left|ab\right|}\enspace.
\end{equation}
\end{lem}
\begin{thm}
The pseudo-likelihood contribution from an observation $y_{n}$ can be expressed as
\begin{equation}
L\left(y_{n}\ |\ \mathbf{x}_{n},f\right)=\int_{0}^{\infty}\frac{1}{\sqrt{2\pi\lambda_{n}}}\exp\left(-\frac{1}{2}\frac{\left(1+\lambda_{n}+\max_{t\neq y_{n},t\in Y}f_{t}\left(\mathbf{x}_{n}\right)-f_{y_{n}}\left(\mathbf{x}_{n}\right)\right)^{2}}{\lambda_{n}}\right)\mathrm{d}\lambda_{n}
\end{equation}
\end{thm}
\begin{proof}
Applying Lemma \ref{lem:lemma} while substituting $a=1$ and $b=1+\max_{t\neq y_{n},t\in Y}f_{t}\left(\mathbf{x}_{n}\right)-f_{y_{n}}\left(\mathbf{x}_{n}\right)$,
multiplying through by $e^{-b}$, and using the identity $\max\left\{b,0\right\}=\frac{1}{2}\left(\left|b\right|+b\right)$, we get,
\begin{equation}
\int_{0}^{\infty}\frac{1}{\sqrt{2\pi\lambda_{n}}}\exp\left(-\frac{1}{2}\frac{\left(b+\lambda_{n}\right)^{2}}{\lambda_{n}}\right)\mathrm{d}\lambda_{n}=e^{-2\max\left\{ b,0\right\} }\enspace .
\end{equation}
\end{proof}
\paragraph{Inference}
We complete the model formulation by assuming that $f_j(\mathbf{x})$ is drawn from a Gaussian process for each class, $j$, i.e. $\mathbf{f}_j \sim \mathcal{N}(0,K_{NN})$ and $\boldsymbol{\lambda} \sim \mathbbm{1}_{[0,\infty)}(\boldsymbol{\lambda})$.
Inference in our model amounts to learning the joint posterior $p(\mathbf{f},\boldsymbol{\lambda}|D)$, where $\mathbf{f} = [\mathbf{f}_1,\dots,\mathbf{f}_C]$.
However, computing the exact posterior is intractable and hence various schemes can be adopted to approximate it.
In our approach we use variational inference combined with an inducing point approximation for scalable learning.

\subsection{Scalable Variational Inference with Inducing Points}
\label{sub:inference}

In variational inference, the exact posterior over the set of model parameters $\boldsymbol{\theta}$ is approximated by a variational distribution $q$.
The parameters of $q$ are updated with the aim to reduce the dissimilarity between the exact and approximate posteriors, as measured by the Kullback-Leibler divergence.
This is equivalent to maximizing the evidence lower bound (ELBO) ~\cite{Jordan1999} with respect to parameters of $q$.
\begin{equation}
\label{eq:elbo}
\operatorname{ELBO} = \mathbb{E}_{q(\boldsymbol{\theta})}\left[\log p\left(\mathbf{y}|\boldsymbol{\theta}\right)\right] - \operatorname{KL}\left[q\left(\boldsymbol{\theta}\right)||p\left(\boldsymbol{\theta}\right)\right]
\end{equation}
Using this objective function, we could potentially infer the posterior $q(\mathbf{f},\boldsymbol{\lambda})$.
However, inference and prediction using this full model involves inverting an $N\times N$ matrix.
An operation of complexity $O(N^3)$ is impractical.
Therefore, we employ the sparse approximation proposed by \citet{Hensman2015}.
We augment the model with $P\ll N$ inducing points which are shared across all Gaussian processes.
Similar to \citet{Hensman2015}, we consider a Gaussian process prior for the inducing points, $p(\mathbf{u}_j) = \mathcal{N}\left(0,K_{PP}\right)$ and consider the marginal $q(\mathbf{f}_j) = \int p(\mathbf{f}_j|\mathbf{u}_j)q(\mathbf{u}_j)\mathrm{d}\mathbf{u}_j$ with $p(\mathbf{f}_j|\mathbf{u}_j) = \mathcal{N}\left(\kappa \mathbf{u},\tilde{K}\right)$.
The approximate posterior $q(\mathbf{u},\boldsymbol{\lambda})$ factorizes as $\prod_{j\in Y}q(\mathbf{u}_j)\prod_{n=1}^Nq(\lambda_n)$ with $q(\lambda_n) = \mathcal{GIG}(1/2,1,\alpha_n)$ and $q(\mathbf{u}_j) = \mathcal{N}(\boldsymbol{\mu}_j,\Sigma_j)$. 
Here, $\kappa=K_{NP}K^{-1}_{PP}$, $\tilde{K}=K_{NN}-K_{NP}\kappa^{T}$ and $\mathcal{GIG}$ is the generalized inverse Gaussian.
$K_{PP}$ is the kernel matrix resulting from evaluating the kernel function between all inducing points.
$K_{NP}$ or $K_{NN}$ are accordingly defined.
The choice of variational approximations is inspired from the exact conditional posterior computed by \citet{Henao2014}.
With an application of Jensen's inequality to Equation \eqref{eq:elbo}, we derive the final training objective,
\begin{align}
& \mathbb{E}_{q(\mathbf{u},\boldsymbol{\lambda})}\left[\log p\left(\mathbf{y}|\mathbf{u},\boldsymbol{\lambda} \right)\right] - \operatorname{KL}\left[q\left(\mathbf{u},\boldsymbol{\lambda} \right)||p\left(\mathbf{u},\boldsymbol{\lambda} \right)\right]\\
\geq & \mathbb{E}_{q(\mathbf{u},\boldsymbol{\lambda})}\left[\mathbb{E}_{p(\mathbf{f}|\mathbf{u})}\left[\log p\left(\mathbf{y},\boldsymbol{\lambda}|\mathbf{f}\right)\right]\right]  + \mathbb{E}_{q(\mathbf{u})}[\log p\left(\mathbf{u}\right)] - \mathbb{E}_{q(\mathbf{u},\boldsymbol{\lambda})}[\log q(\mathbf{u},\boldsymbol{\lambda})]\\
= & \sum_{n=1}^{N}\left(-\frac{1}{2\sqrt{\alpha_{n}}}\left(2\tilde{K}_{n,n}+\left(1+\boldsymbol{\kappa}_{n}\left(\boldsymbol{\mu}_{t_{n}}-\boldsymbol{\mu}_{y_{n}}\right)\right)^{2}+\boldsymbol{\kappa}_{n}\Sigma_{t_n}\boldsymbol{\kappa}_{n}^{\intercal}+\boldsymbol{\kappa}_{n}\Sigma_{y_n}\boldsymbol{\kappa}_{n}^{\intercal}-\alpha_{n}\right)\right.\nonumber\\
 & \phantom{\sum_{i=1}^{n}\left(\right)}\left.-\boldsymbol{\kappa}_{n}\left(\boldsymbol{\mu}_{t_{n}}-\boldsymbol{\mu}_{y_{n}}\right)-\frac{1}{4}\log\alpha_{n}-\log\left(B_{\frac{1}{2}}\left(\sqrt{\alpha_{n}}\right)\right)\right)\nonumber\\
 & \phantom{\sum_{i=1}^{n}\left(\right)}-\frac{1}{2}\sum_{j\in Y}\left(-\log\left|\Sigma_{j}\right|+\text{trace}\left(K^{-1}_{PP}\Sigma_{j}\right)+\boldsymbol{\mu}_{j}^{\intercal}K^{-1}_{PP}\boldsymbol{\mu}_{j}\right) = \mathcal{O}\enspace,
\end{align}
where $B_{\frac{1}{2}}$ is the modified Bessel function~\cite{Joergensen1982}, and $t_{n}=\argmax_{t\in Y,t\neq y_{n}}f_{t}\left(\mathbf{x}_{n}\right)$.
$\mathcal{O}$ is maximized using gradient-based optimization methods.
We provide a detailed derivation of the variational objective and its gradients in the appendix.





\subsection{Hybrid Bayesian Neural Networks}\label{sub:hybrid-bnn}
\begin{figure}
  \centering
  \includegraphics[width=0.7\textwidth]{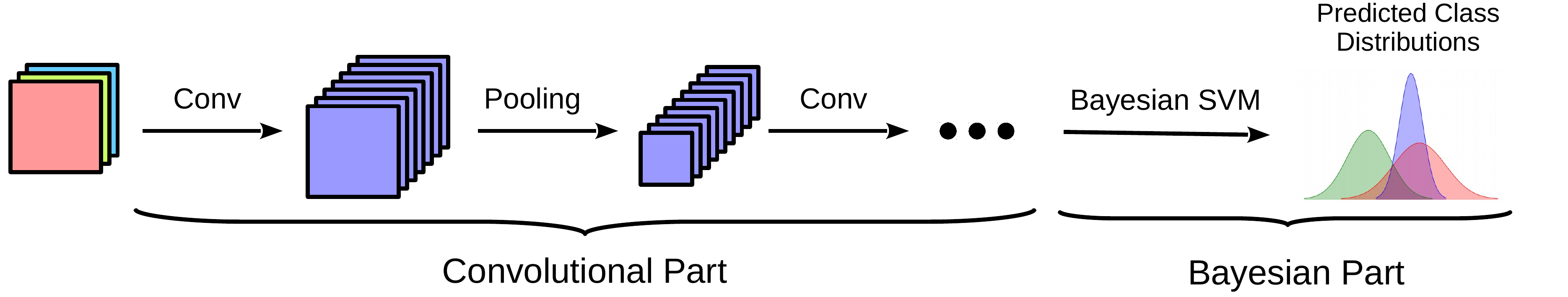}
  \caption{Hybrid Bayesian neural network with a Bayesian SVM for image classification.\label{fig:hybrid_nn}}
\end{figure}
In Section \ref{sub:inference} we show that our proposed multi-class Bayesian SVM can be learned with gradient-based optimization schemes.
This enables us to combine it with various deep learning components such as convolutional layers and extend its applicability to unstructured data as shown in Figure \ref{fig:hybrid_nn}.
The parameters of the convolution and the variational parameters are jointly learned by means of backpropagation.
\begin{figure}
  \includegraphics[width=0.43\textwidth]{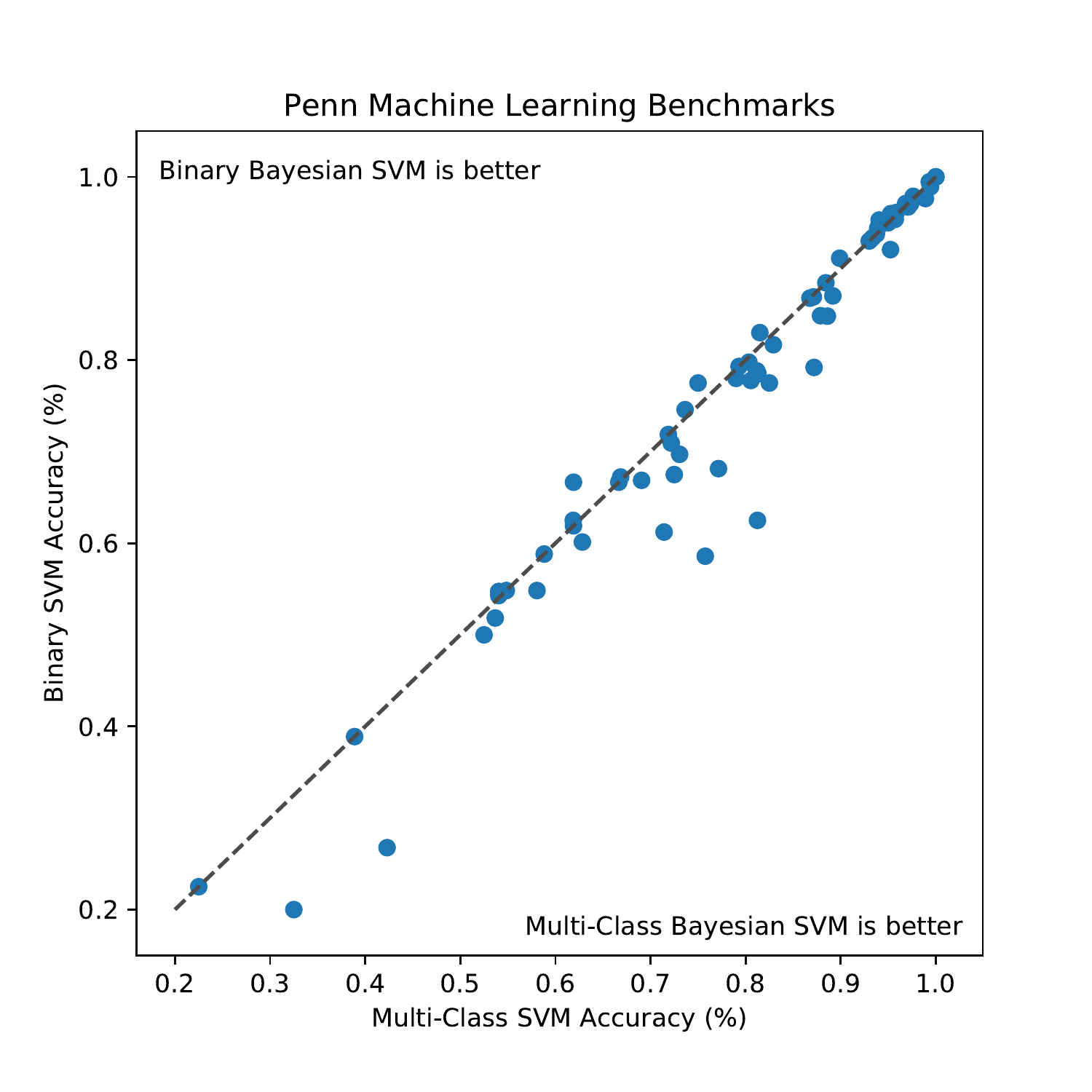}
  \hfill
  \includegraphics[width=0.43\textwidth]{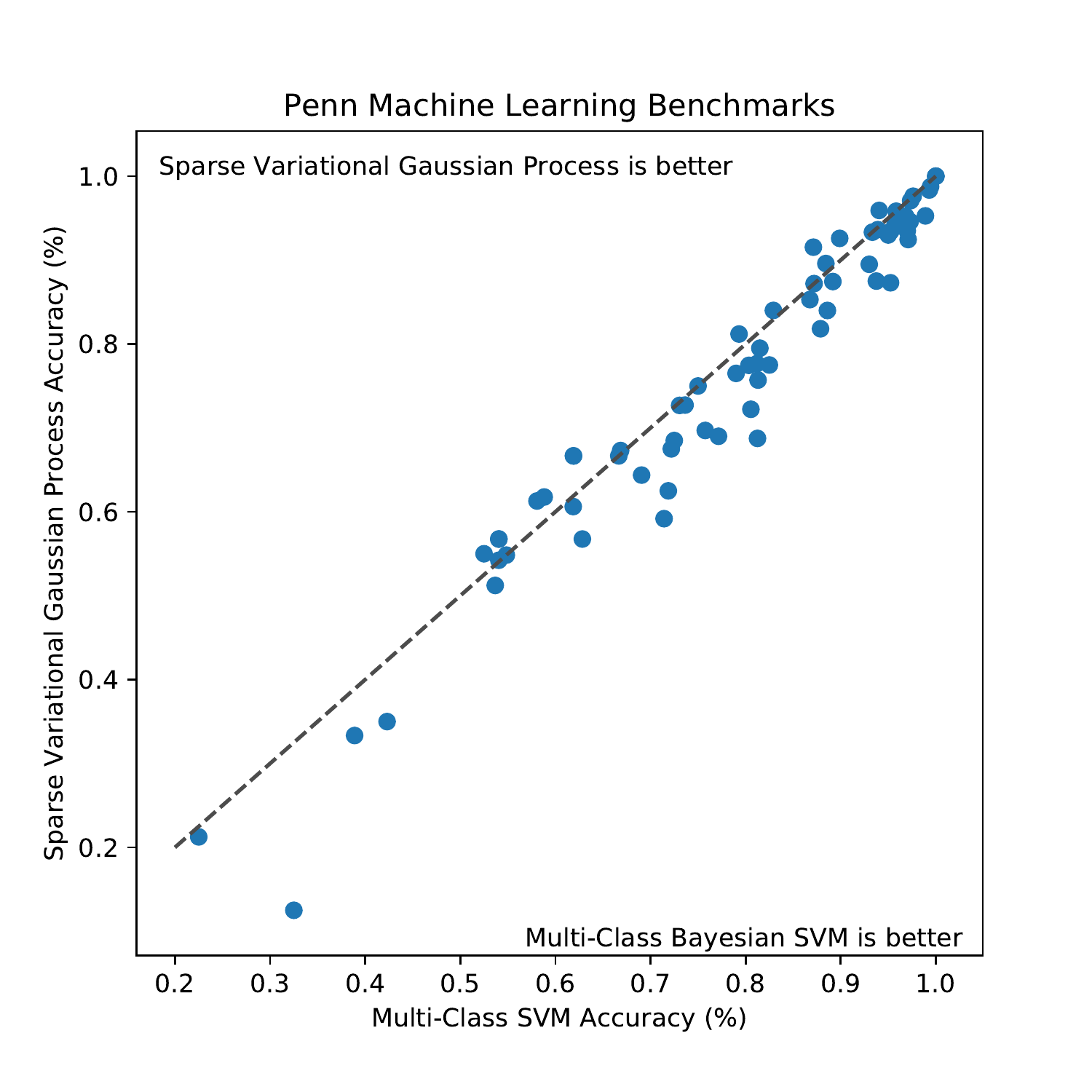}
  \caption{Pairwise comparison of the multi-class Bayesian SVM versus the binary Bayesian SVM and SVGP. On average, our proposed multi-class Bayesian SVM provides better results.\label{fig:classification}}
\end{figure}

\section{Experimental Evaluation}

In this section we conduct an extensive study of the multi-class Bayesian SVM and analyze its classification performance on structured and unstructured data.
Additionally, we analyze the quality of its uncertainty prediction in a large-scale active learning experiment and for the challenging problem of adversarial image detection.

%

\subsection{Classification of Structured Data}\label{sub:structured-data}

\begin{table}[t]
  \caption{Mean average rank across 68 data sets. The smaller, the better. Our proposed multi-class Bayesian SVM is on average the most accurate prediction model.}
  \label{tab:classification}
  \centering
  \begin{tabular}{rrr}
    \toprule
    Binary Bayesian SVM     & Multi-Class Bayesian SVM     & Scalable Variational Gaussian Process \\
    \midrule
    1.96 & \textbf{1.68}  & 2.33     \\
    \bottomrule
  \end{tabular}
\end{table}
We evaluate the proposed multi-class Bayesian support vector machine with respect to classification accuracy on the Penn Machine Learning Benchmarks~\cite{Olson2017}.
From this benchmark, we select all multi-class classification data sets consisting of at least 128 instances.
This subset consists of 68 data sets with up to roughly one million instances.
We compare the classification accuracy of our proposed multi-class Bayesian SVM with the most recently proposed binary Bayesian support vector machine~\cite{Wenzel2017} (one-vs-rest setup) and the scalable variational Gaussian process (SVGP)~\cite{Hensman2015}.
We use the implementation available in GPflow~\cite{GPflow2017} for SVGP and implement the binary and multi-class Bayesian SVM as additional classifiers in GPflow by extending its classifier interface.
The shared back end of all three implementations allows a fair training time comparison.
For this experiment, all models are trained using 64 inducing points.
Gradient-based optimization is performed using Adam~\cite{Kingma2014} with an initial learning rate of $5\cdot10^{-4}$ for 1000 epochs.

Figure~\ref{fig:classification} contrasts the multi-class Bayesian SVM with its binary counterpart and SVGP.
The proposed multi-class Bayesian SVM clearly outperforms the other two models for most data sets.
While this is more pronounced against SVGP, the binary and multi-class Bayesian SVMs exhibit similar behavior.
This claim is supported by the comparison of mean ranks (Table~\ref{tab:classification}).
The rank per data set is computed by ranking the methods for each data set according to classification accuracy.
The most accurate prediction model is assigned rank 1, second best rank 2 and so on.
In case of ties, an average rank is used, e.g. if the models exhibit classification accuracies of 1.0, 1.0, and 0.8, they are assigned ranks of 1.5, 1.5, and 3, respectively.

\begin{figure}[t]
  \includegraphics[width=1\textwidth]{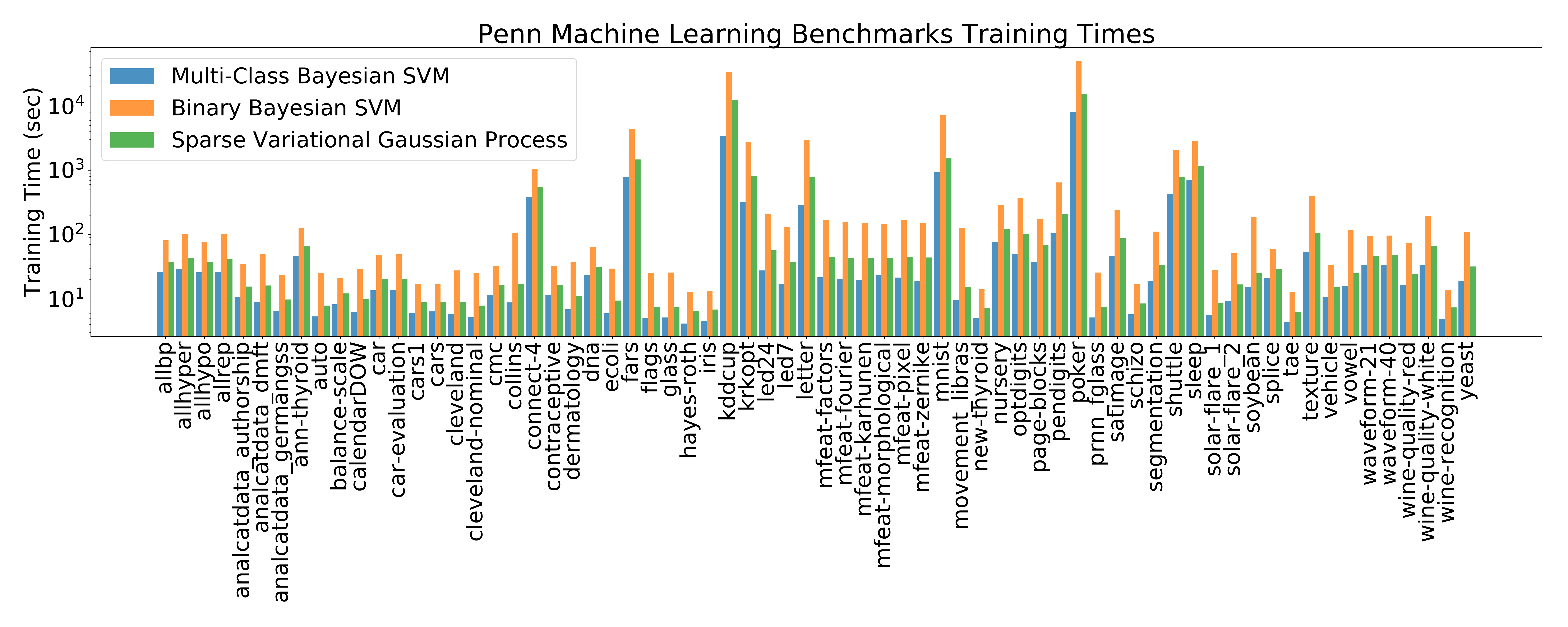}
  \caption{Our proposed multi-class Bayesian SVM clearly needs less time than its competitors.}\label{fig:classification_time}
\end{figure}
The primary motivation for proposing the multi-class Bayesian SVM is scalability.
Classification using the binary Bayesian SVM requires training an independent model per class which increases the training time by a factor equal to the number of classes.
Contrastingly, SVGP and our proposed model enable multi-class classification with a single model.
This results in significant benefits in training time.
As evident in Figure~\ref{fig:classification_time}, the multi-class Bayesian SVM requires the least training time.

In conclusion, the multi-class Bayesian SVM is the most efficient model without compromising on prediction accuracy.
In fact, on average it has a higher accuracy.

\subsection{Classification of Image Data}\label{sub:image-classification}
\begin{figure}[t]
  \centering
  \includegraphics[width=0.44\textwidth]{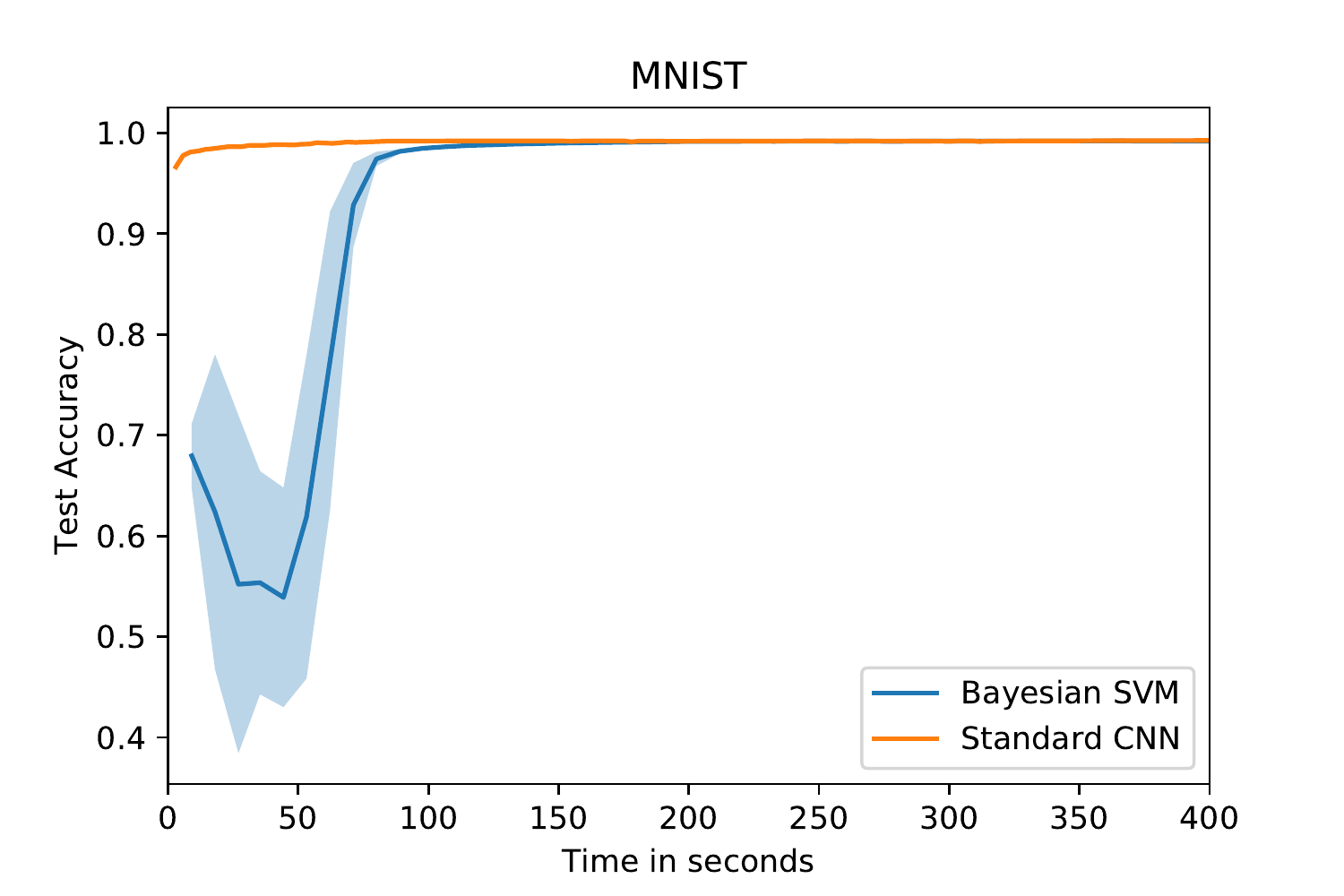}
  \hspace{1cm}
  \includegraphics[width=0.44\textwidth]{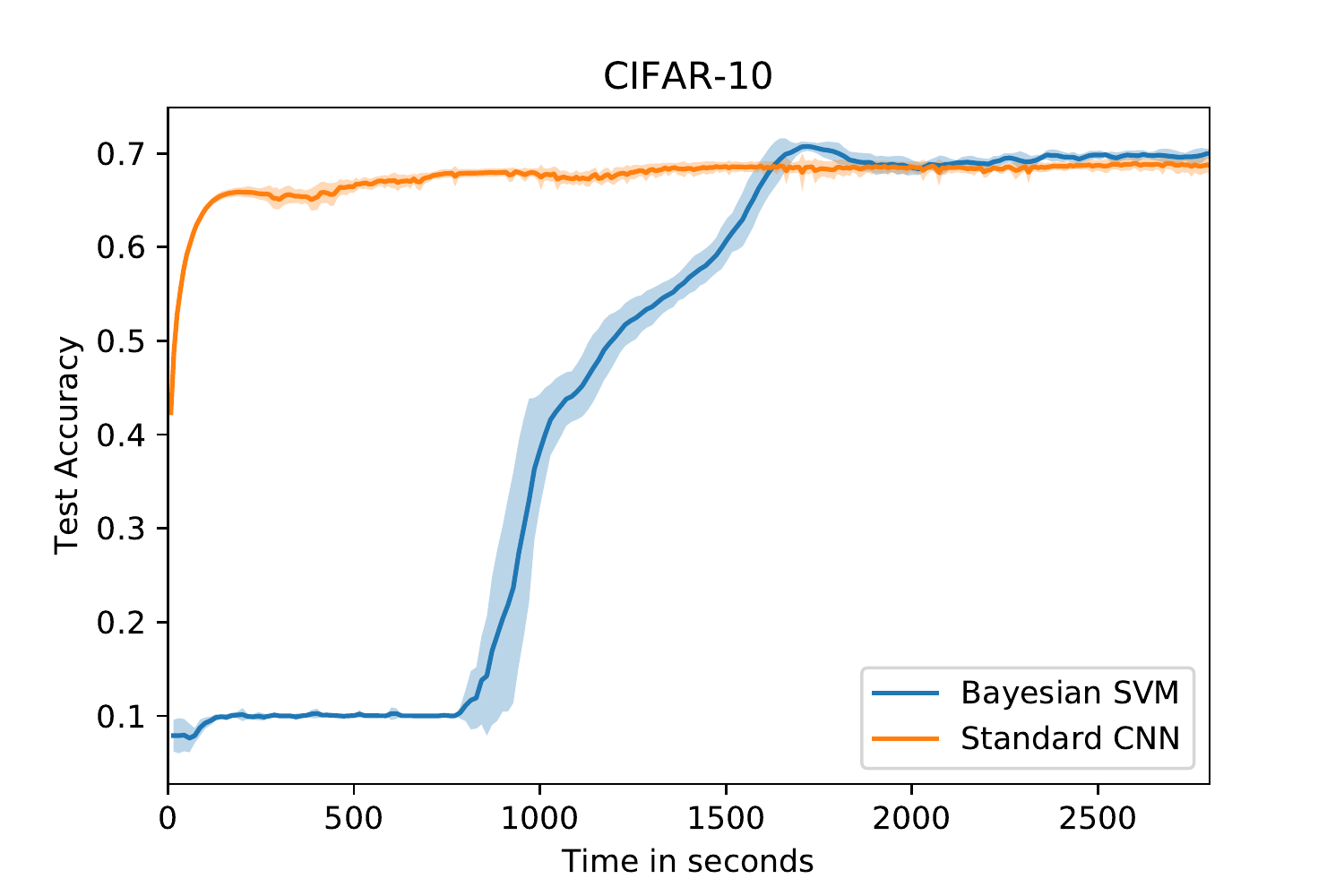}
  \caption{The jointly learned model of a convolutional network and a Bayesian SVM performs as good as a standard network. The price of gaining a Bayesian neural network is a longer training time.}
  \label{fig:image_classification}
\end{figure}
In Section \ref{sub:hybrid-bnn} we describe how deep learning can be used to learn a feature representation jointly with a multi-class Bayesian SVM.
Image data serves as a typical example for unstructured data.
We compare the hybrid Bayesian neural network to a standard convolutional neural network (CNN) with a softmax layer for classification.
We evaluate these models on two popular image classification benchmarks, MNIST~\cite{MNIST} and CIFAR-10~\cite{CIFAR10}.

We observe same performance of the hybrid Bayesian neural network using the Bayesian SVM as a standard CNN with softmax layer.
The two different neural networks share the first set of layers, for MNIST: \texttt{conv(32,5,5)-conv(64,3,3)-max\_pool-fc(1024)-fc(100)}, and for CIFAR-10: \texttt{conv(128,3,3)-conv(128,3,3)-max\_pool-conv(128,3,3)-max\_pool-fc(256)-fc(100)}.
As in our previous experiment, we use Adam to perform the optimization.

Figure~\ref{fig:image_classification} shows that the hybrid Bayesian neural network achieves the same test accuracy as the standard CNN.
The additional training effort of a hybrid Bayesian neural network pays off in achieving probabilistic predictions with uncertainty estimates.
While the variational objective and the likelihood exhibits the expected behavior during the training, we note an odd behavior during the initial epochs.
We suspect that this is due to initialization of parameters which could result in the KL-term of the variational objective dominating the expected log-likelihood.

\subsection{Uncertainty Analysis}
One of the advantages of using Bayesian machine learning models lies in getting a distribution over predictions rather than just point-estimates.
In this section, we demonstrate this advantage in the domains of active learning and adversarial image detection.

\subsubsection{Active Learning}
\begin{figure}[t]
  \centering
  \subfloat[Average rank across 68 data sets.\label{fig:active_learning_summary}]
  {\includegraphics[width=0.44\textwidth]{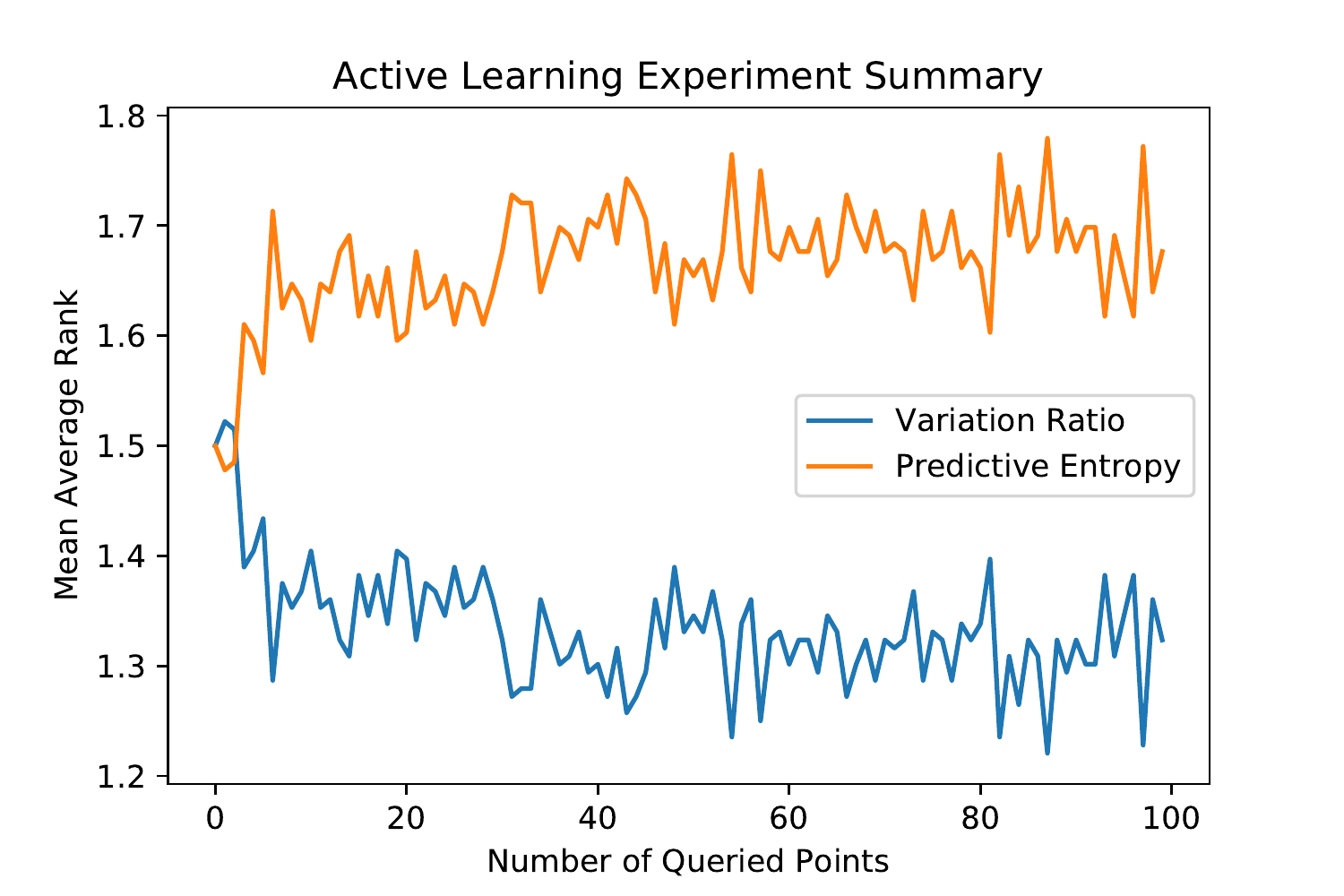}}
  \hspace{1cm}
  \subfloat[Representative results for the largest data set.\label{fig:active_learning_poker}]
  {\includegraphics[width=0.44\textwidth]{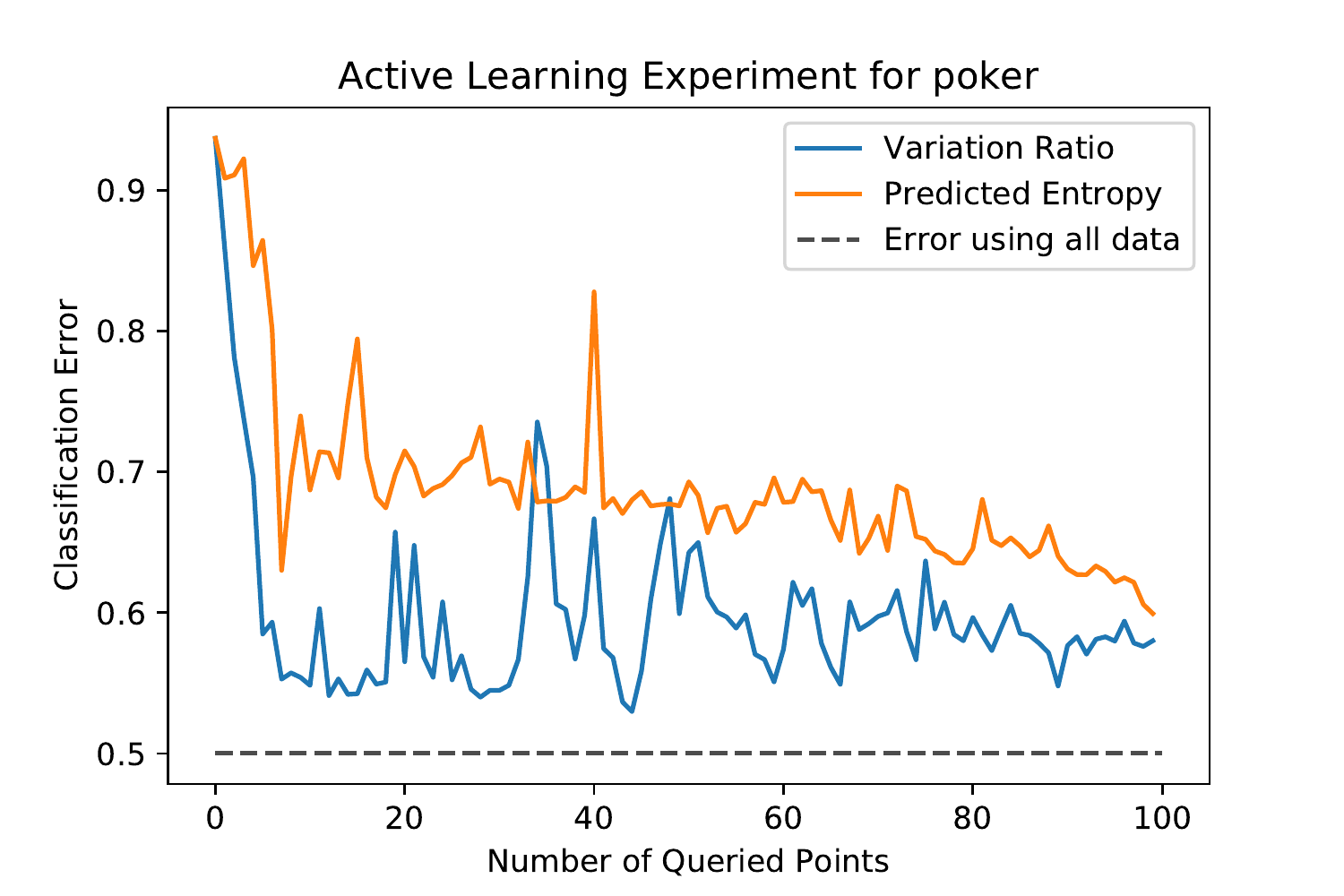}}
  \caption{The Bayesian query policy (variation ratio) decreases the error of the model faster and clearly outperforms the policy based on point-estimates only. For both figures, the smaller the better.}
\end{figure}
Active learning is concerned with scenarios where the process of labeling data is expensive.
In such scenarios, a query policy is adopted to label samples from a large pool of unlabeled instances with the aim to improve model performance.
We contrast between two policies to highlight the merits of using prediction uncertainty obtained from the multi-class Bayesian SVM.
While the first policy utilizes both mean and variance of the predictive distribution of the Bayesian SVM, the second policy relies only on the mean.
For this experiment we use the same data sets as specified in Section~\ref{sub:structured-data}.

We use the variation ratio (VR) as the basis of a Bayesian query policy.
It is defined by
\begin{equation}
 \text{Variation Ratio} = 1 - F/S\enspace,\label{eq:variation-ratio}
\end{equation}
where $F$ is the frequency of the mode and $S$ the number of samples.
The VR is the relative number of times the majority class is not predicted.
The instance with highest VR is queried.
We compare this to a policy which queries the instance with maximum entropy of class probabilities.
These are computed using softmax over the mean predictions.
For a fair comparison, we the same multi-class Bayesian SVM for both policies.
Initially, one instance per class, selected uniformly at random, is labeled.
Then, one hundred further instances are queried according to the two policies.
As only few training examples are available, we modify the training setup by reducing the number of inducing points to four.

We report the mean average rank across 68 data sets for the two different query policies in Figure~\ref{fig:active_learning_summary}.
Since both policies start with the same set of labeled instances, the performance is very similar at the beginning.
However, with increasing number of queried data points, the Bayesian policy quickly outperforms the other policy.
Of the 68 data sets, the \textit{poker} data set, with more than one million instances, is the largest and consequently the most challenging.
Within the first queries, we observe a large decrease in classification error as shown in Figure~\ref{fig:active_learning_poker}.
We note the same trend of mean ranks across the two policies.
The small number of labeled instances is obviously insufficient to reach the test error of a model trained on all data points as shown by the dashed line.
\begin{figure}[t]
  \centering
  \includegraphics[width=0.44\textwidth]{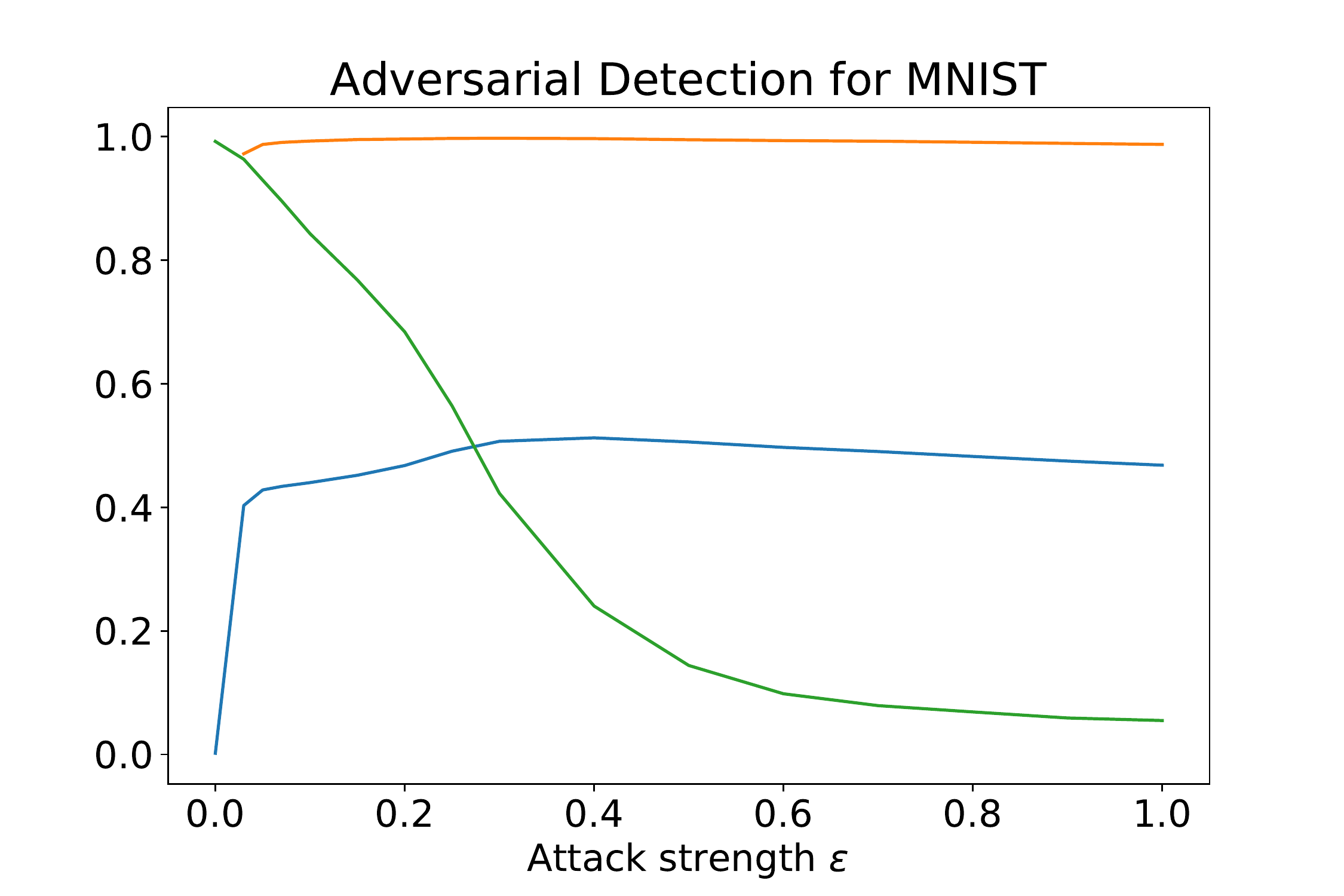}
  \hspace{1cm}
  \includegraphics[width=0.44\textwidth]{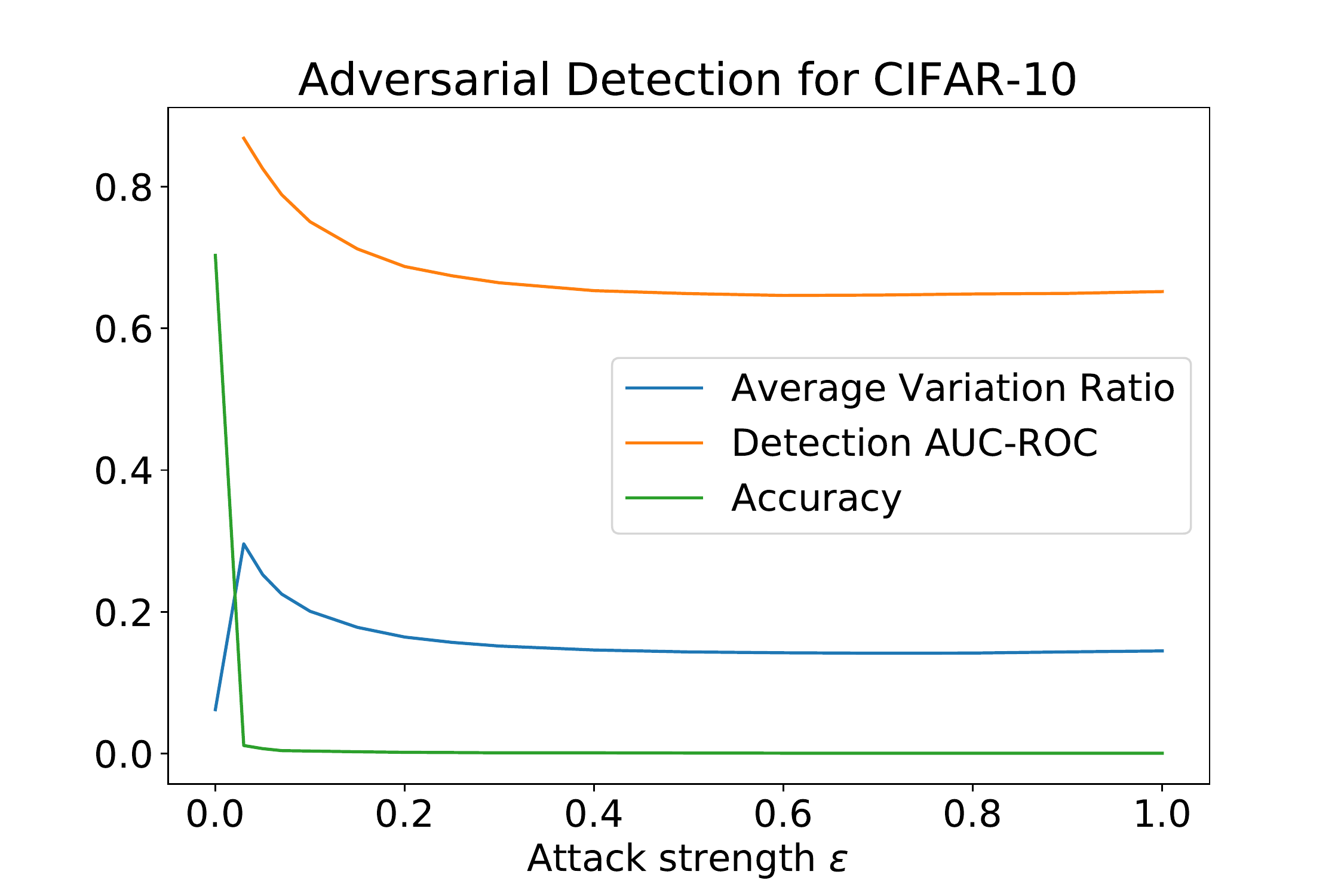}
  \caption{The accuracy on adversarial images decreases with increasing attack strength.
  A significant increase of the average variation ratio indicates that it is a good feature to detect adversarial images.
  }
  \label{fig:adversarial_detection}
\end{figure}

\subsubsection{Adversarial Image Detection}
With the rise of Deep Learning, its security and reliability is a major concern.
A recent development in this direction is the discovery of adversarial images~\cite{Goodfellow2014}.
These correspond to images obtained by adding small imperceptible adversarial noise resulting in high confidence misclassification.
While various successful attacks exist, most defense and detection methods do not work~\cite{Carlini2017}.
However, \citet{Carlini2017} acknowledge that the uncertainty obtained from Bayesian machine learning models is the most promising research direction.
Several studies show that Bayesian models behave differently for adversarial examples compared to the original data~\cite{Bradshaw2017,Li2017,Smith2018,Rawat2017}.
We take a step further and use the variation ratio (VR) determined by the Bayesian SVM, as defined in Equation~\eqref{eq:variation-ratio}, for building a detection model for adversarial images.

We attack the hybrid Bayesian neural network described in Section~\ref{sub:image-classification} with the popular Fast Gradient Sign Method (FGSM)~\cite{Goodfellow2014}.
We generate one adversarial image per image in the test set.
We present the results for detection and classification under attack in Figure~\ref{fig:adversarial_detection}.
The Bayesian SVM is not robust to FGSM since its accuracy drops with increasing attack strength $\epsilon$.
However, the attack does not remain unperceived.
The VR rapidly increases and enables the detection of adversarial images.
The ranking of original and adversarial examples with respect to VR yields an ROC-AUC of almost 1 for MNIST.
This means that the VR computed for any original example is almost always smaller than the one computed for any adversarial example.

CIFAR-10 exhibits different results under the same setup.
Here, the detection is poor and it significantly worsens with increasing attack strength.
Potentially, this is an artifact of the poor classification model for CIFAR-10.
In contrast to the MNIST classifier, this model is under-confident on original examples.
Thus, a weaker attack succeeds in reducing the test accuracy to 1.16\%.
We believe a better network architecture combined with techniques such as data augmentation will lead to an improved performance in terms of test accuracy and subsequently better detection.
Nevertheless, the detection performance of our model is still better than a random detector, even for the strongest attack.

\section{Conclusions}

We devise a pseudo-likelihood for the multi-class hinge loss leading to the first multi-class Bayesian support vector machine.
Additionally, we derive a variational training objective for the proposed model and develop a scalable inference algorithm to optimize it.
We establish the efficacy of the model on multi-class classification tasks with extensive experimentation on structured data and contrast its accuracy to two state-of-the-art competitor methods.
We provide empirical evidence that our proposed method is on average better and up to an order of magnitude faster to train.
Furthermore, we extend our formulation to a hybrid Bayesian neural network and report comparable accuracy to standard models for image classification tasks.
Finally, we investigate the key advantage of Bayesian modeling in our approach by demonstrating the use of prediction uncertainty in solving the challenging tasks of active learning and adversarial image detection.
The uncertainty-based policy outperforms its competitor in the active learning scenario.
Similarly, the uncertainty-enabled adversarial detection shows promising results for image data sets with near-perfect performance on MNIST.

\bibliographystyle{plainnat}
\bibliography{myrefs}

\begin{thebibliography}{36}
\providecommand{\natexlab}[1]{#1}
\providecommand{\url}[1]{\texttt{#1}}
\expandafter\ifx\csname urlstyle\endcsname\relax
  \providecommand{\doi}[1]{doi: #1}\else
  \providecommand{\doi}{doi: \begingroup \urlstyle{rm}\Url}\fi

\bibitem[Amari and Nagaoka(2007)]{Amari2007}
Shun-ichi Amari and Hiroshi Nagaoka.
\newblock \emph{Methods of information geometry}, volume 191.
\newblock American Mathematical Soc., 2007.

\bibitem[Andrews and Mallows(1974)]{Andrews1974}
D.~F. Andrews and C.~L. Mallows.
\newblock Scale mixtures of normal distributions.
\newblock \emph{Journal of the Royal Statistical Society. Series B
  (Methodological)}, 36\penalty0 (1):\penalty0 99--102, 1974.
\newblock ISSN 00359246.
\newblock URL \url{http://www.jstor.org/stable/2984774}.

\bibitem[Blei et~al.(2016)Blei, Kucukelbir, and McAuliffe]{Blei2016}
David~M. Blei, Alp Kucukelbir, and Jon~D. McAuliffe.
\newblock Variational inference: {A} review for statisticians.
\newblock \emph{CoRR}, abs/1601.00670, 2016.
\newblock URL \url{http://arxiv.org/abs/1601.00670}.

\bibitem[Boser et~al.(1992)Boser, Guyon, and Vapnik]{Boser1992}
Bernhard~E. Boser, Isabelle Guyon, and Vladimir Vapnik.
\newblock A training algorithm for optimal margin classifiers.
\newblock In \emph{Proceedings of the Fifth Annual {ACM} Conference on
  Computational Learning Theory, {COLT} 1992, Pittsburgh, PA, USA, July 27-29,
  1992.}, pages 144--152, 1992.
\newblock \doi{10.1145/130385.130401}.
\newblock URL \url{http://doi.acm.org/10.1145/130385.130401}.

\bibitem[Bradshaw et~al.(2017)Bradshaw, de~G.~Matthews, and
  Ghahramani]{Bradshaw2017}
John Bradshaw, Alexander~G. de~G.~Matthews, and Zoubin Ghahramani.
\newblock Adversarial examples, uncertainty, and transfer testing robustness in
  gaussian process hybrid deep networks, 2017.

\bibitem[Carlini and Wagner(2017)]{Carlini2017}
Nicholas Carlini and David Wagner.
\newblock Adversarial examples are not easily detected: Bypassing ten detection
  methods, 2017.

\bibitem[Cortes and Vapnik(1995)]{Cortes1995}
Corinna Cortes and Vladimir Vapnik.
\newblock Support-vector networks.
\newblock \emph{Machine Learning}, 20\penalty0 (3):\penalty0 273--297, 1995.
\newblock \doi{10.1007/BF00994018}.
\newblock URL \url{https://doi.org/10.1007/BF00994018}.

\bibitem[Crammer and Singer(2001)]{Crammer2001}
Koby Crammer and Yoram Singer.
\newblock On the algorithmic implementation of multiclass kernel-based vector
  machines.
\newblock \emph{Journal of Machine Learning Research}, 2:\penalty0 265--292,
  2001.
\newblock URL \url{http://www.jmlr.org/papers/v2/crammer01a.html}.

\bibitem[Dogan et~al.(2016)Dogan, Glasmachers, and Igel]{Dogan2016}
{\"{U}}r{\"{u}}n Dogan, Tobias Glasmachers, and Christian Igel.
\newblock A unified view on multi-class support vector classification.
\newblock \emph{Journal of Machine Learning Research}, 17:\penalty0
  45:1--45:32, 2016.
\newblock URL \url{http://jmlr.org/papers/v17/11-229.html}.

\bibitem[Goodfellow et~al.(2014)Goodfellow, Shlens, and
  Szegedy]{Goodfellow2014}
Ian~J. Goodfellow, Jonathon Shlens, and Christian Szegedy.
\newblock Explaining and harnessing adversarial examples, 2014.

\bibitem[Henao et~al.(2014)Henao, Yuan, and Carin]{Henao2014}
Ricardo Henao, Xin Yuan, and Lawrence Carin.
\newblock Bayesian nonlinear support vector machines and discriminative factor
  modeling.
\newblock In \emph{Advances in Neural Information Processing Systems}, pages
  1754--1762, 2014.

\bibitem[Hensman et~al.(2015)Hensman, de~G.~Matthews, and
  Ghahramani]{Hensman2015}
James Hensman, Alexander~G. de~G.~Matthews, and Zoubin Ghahramani.
\newblock Scalable variational gaussian process classification.
\newblock In \emph{Proceedings of the Eighteenth International Conference on
  Artificial Intelligence and Statistics, {AISTATS} 2015, San Diego,
  California, USA, May 9-12, 2015}, 2015.
\newblock URL \url{http://jmlr.org/proceedings/papers/v38/hensman15.html}.

\bibitem[Hofmann et~al.(2008)Hofmann, Sch{\"o}lkopf, and Smola]{Hofmann2008}
T.~Hofmann, B.~Sch{\"o}lkopf, and AJ. Smola.
\newblock Kernel methods in machine learning.
\newblock \emph{Annals of Statistics}, 36\penalty0 (3):\penalty0 1171--1220,
  June 2008.

\bibitem[Jones et~al.(1998)Jones, Schonlau, and Welch]{Jones1998}
Donald~R. Jones, Matthias Schonlau, and William~J. Welch.
\newblock Efficient global optimization of expensive black-box functions.
\newblock \emph{J. Global Optimization}, 13\penalty0 (4):\penalty0 455--492,
  1998.
\newblock \doi{10.1023/A:1008306431147}.
\newblock URL \url{https://doi.org/10.1023/A:1008306431147}.

\bibitem[Jordan et~al.(1999)Jordan, Ghahramani, Jaakkola, and Saul]{Jordan1999}
Michael~I. Jordan, Zoubin Ghahramani, Tommi~S. Jaakkola, and Lawrence~K. Saul.
\newblock An introduction to variational methods for graphical models.
\newblock \emph{Machine Learning}, 37\penalty0 (2):\penalty0 183--233, 1999.
\newblock \doi{10.1023/A:1007665907178}.
\newblock URL \url{https://doi.org/10.1023/A:1007665907178}.

\bibitem[Jørgensen(1982)]{Joergensen1982}
{Bent} Jørgensen.
\newblock \emph{Statistical properties of the generalized inverse Gaussian
  distribution}.
\newblock Number~9 in Lecture notes in statistics. Springer, New York, NY
  [u.a.], 1982.
\newblock ISBN 0387906657.

\bibitem[Kingma and Ba(2014)]{Kingma2014}
Diederik~P. Kingma and Jimmy Ba.
\newblock Adam: {A} method for stochastic optimization.
\newblock \emph{CoRR}, abs/1412.6980, 2014.
\newblock URL \url{http://arxiv.org/abs/1412.6980}.

\bibitem[Krizhevsky(2009)]{CIFAR10}
Alex Krizhevsky.
\newblock Learning multiple layers of features from tiny images.
\newblock Technical report, 2009.

\bibitem[Kuss and Rasmussen(2005)]{Kuss2005}
Malte Kuss and Carl~Edward Rasmussen.
\newblock Assessing approximate inference for binary gaussian process
  classification.
\newblock \emph{Journal of Machine Learning Research}, 6:\penalty0 1679--1704,
  2005.
\newblock URL \url{http://www.jmlr.org/papers/v6/kuss05a.html}.

\bibitem[LeCun and Cortes(2010)]{MNIST}
Yann LeCun and Corinna Cortes.
\newblock {MNIST} handwritten digit database.
\newblock 2010.
\newblock URL \url{http://yann.lecun.com/exdb/mnist/}.

\bibitem[Li et~al.(2010)Li, Chu, Langford, and Schapire]{Li2010}
Lihong Li, Wei Chu, John Langford, and Robert~E. Schapire.
\newblock A contextual-bandit approach to personalized news article
  recommendation.
\newblock In \emph{Proceedings of the 19th International Conference on World
  Wide Web, {WWW} 2010, Raleigh, North Carolina, USA, April 26-30, 2010}, pages
  661--670, 2010.
\newblock \doi{10.1145/1772690.1772758}.
\newblock URL \url{http://doi.acm.org/10.1145/1772690.1772758}.

\bibitem[Li and Gal(2017)]{Li2017}
Yingzhen Li and Yarin Gal.
\newblock Dropout inference in bayesian neural networks with alpha-divergences.
\newblock In \emph{Proceedings of the 34th International Conference on Machine
  Learning, {ICML} 2017, Sydney, NSW, Australia, 6-11 August 2017}, pages
  2052--2061, 2017.
\newblock URL \url{http://proceedings.mlr.press/v70/li17a.html}.

\bibitem[Luts and Ormerod(2014)]{Luts2014}
Jan Luts and John~T. Ormerod.
\newblock Mean field variational bayesian inference for support vector machine
  classification.
\newblock \emph{Computational Statistics {\&} Data Analysis}, 73:\penalty0
  163--176, 2014.
\newblock \doi{10.1016/j.csda.2013.10.030}.
\newblock URL \url{https://doi.org/10.1016/j.csda.2013.10.030}.

\bibitem[Matthews et~al.(2017)Matthews, {van der Wilk}, Nickson, Fujii,
  {Boukouvalas}, {Le{\'o}n-Villagr{\'a}}, Ghahramani, and Hensman]{GPflow2017}
Alexander G. de~G. Matthews, Mark {van der Wilk}, Tom Nickson, Keisuke. Fujii,
  Alexis {Boukouvalas}, Pablo {Le{\'o}n-Villagr{\'a}}, Zoubin Ghahramani, and
  James Hensman.
\newblock {{GP}flow: A {G}aussian process library using {T}ensor{F}low}.
\newblock \emph{Journal of Machine Learning Research}, 18\penalty0
  (40):\penalty0 1--6, apr 2017.
\newblock URL \url{http://jmlr.org/papers/v18/16-537.html}.

\bibitem[Olson et~al.(2017)Olson, La~Cava, Orzechowski, Urbanowicz, and
  Moore]{Olson2017}
Randal~S. Olson, William La~Cava, Patryk Orzechowski, Ryan~J. Urbanowicz, and
  Jason~H. Moore.
\newblock Pmlb: a large benchmark suite for machine learning evaluation and
  comparison.
\newblock \emph{BioData Mining}, 10\penalty0 (1):\penalty0 36, Dec 2017.
\newblock ISSN 1756-0381.
\newblock \doi{10.1186/s13040-017-0154-4}.
\newblock URL \url{https://doi.org/10.1186/s13040-017-0154-4}.

\bibitem[Polson et~al.(2011)Polson, Scott, et~al.]{Polson2011}
Nicholas~G Polson, Steven~L Scott, et~al.
\newblock Data augmentation for support vector machines.
\newblock \emph{Bayesian Analysis}, 6\penalty0 (1):\penalty0 1--23, 2011.

\bibitem[Rasmussen and Williams(2006)]{Rasmussen2006}
Carl~Edward Rasmussen and Christopher K.~I. Williams.
\newblock \emph{Gaussian processes for machine learning}.
\newblock Adaptive computation and machine learning. {MIT} Press, 2006.
\newblock ISBN 026218253X.

\bibitem[Rawat et~al.(2017)Rawat, Wistuba, and Nicolae]{Rawat2017}
Ambrish Rawat, Martin Wistuba, and Maria-Irina Nicolae.
\newblock Adversarial phenomenon in the eyes of bayesian deep learning.
\newblock \emph{arXiv preprint arXiv:1711.08244}, 2017.

\bibitem[Settles(2009)]{Settles2009}
Burr Settles.
\newblock Active learning literature survey.
\newblock Computer Sciences Technical Report 1648, University of
  Wisconsin--Madison, 2009.

\bibitem[Smith and Gal(2018)]{Smith2018}
Lewis Smith and Yarin Gal.
\newblock Understanding measures of uncertainty for adversarial example
  detection.
\newblock \emph{arXiv preprint arXiv:1803.08533}, 2018.

\bibitem[Snelson and Ghahramani(2005)]{Snelson2005}
Edward Snelson and Zoubin Ghahramani.
\newblock Sparse gaussian processes using pseudo-inputs.
\newblock In \emph{Advances in Neural Information Processing Systems 18 [Neural
  Information Processing Systems, {NIPS} 2005, December 5-8, 2005, Vancouver,
  British Columbia, Canada]}, pages 1257--1264, 2005.
\newblock URL
  \url{http://papers.nips.cc/paper/2857-sparse-gaussian-processes-using-pseudo-inputs}.

\bibitem[Snoek et~al.(2012)Snoek, Larochelle, and Adams]{Snoek2012}
Jasper Snoek, Hugo Larochelle, and Ryan~P. Adams.
\newblock Practical bayesian optimization of machine learning algorithms.
\newblock In \emph{Advances in Neural Information Processing Systems 25: 26th
  Annual Conference on Neural Information Processing Systems 2012. Proceedings
  of a meeting held December 3-6, 2012, Lake Tahoe, Nevada, United States.},
  pages 2960--2968, 2012.
\newblock URL
  \url{http://papers.nips.cc/paper/4522-practical-bayesian-optimization-of-machine-learning-algorithms}.

\bibitem[Titsias(2009)]{Titsias2009}
Michalis~K. Titsias.
\newblock Variational learning of inducing variables in sparse gaussian
  processes.
\newblock In \emph{Proceedings of the Twelfth International Conference on
  Artificial Intelligence and Statistics, {AISTATS} 2009, Clearwater Beach,
  Florida, USA, April 16-18, 2009}, pages 567--574, 2009.
\newblock URL \url{http://www.jmlr.org/proceedings/papers/v5/titsias09a.html}.

\bibitem[Wenzel et~al.(2017)Wenzel, Galy-Fajou, Deutsch, and Kloft]{Wenzel2017}
Florian Wenzel, Th{\'e}o Galy-Fajou, Matth{\"a}us Deutsch, and Marius Kloft.
\newblock Bayesian nonlinear support vector machines for big data.
\newblock In \emph{Joint European Conference on Machine Learning and Knowledge
  Discovery in Databases}, pages 307--322. Springer, 2017.

\bibitem[Williams and Barber(1998)]{Williams1998}
Christopher K.~I. Williams and David Barber.
\newblock Bayesian classification with gaussian processes.
\newblock \emph{{IEEE} Trans. Pattern Anal. Mach. Intell.}, 20\penalty0
  (12):\penalty0 1342--1351, 1998.
\newblock \doi{10.1109/34.735807}.
\newblock URL \url{https://doi.org/10.1109/34.735807}.

\bibitem[Wistuba et~al.(2018)Wistuba, Schilling, and
  Schmidt{-}Thieme]{Wistuba2018}
Martin Wistuba, Nicolas Schilling, and Lars Schmidt{-}Thieme.
\newblock Scalable gaussian process-based transfer surrogates for
  hyperparameter optimization.
\newblock \emph{Machine Learning}, 107\penalty0 (1):\penalty0 43--78, 2018.
\newblock \doi{10.1007/s10994-017-5684-y}.
\newblock URL \url{https://doi.org/10.1007/s10994-017-5684-y}.

\end{thebibliography}

\appendix
\section{Derivation of the Variational Training Objective}

We provide a detailed derivation of the proposed variational training objective, 
\begin{align}
\mathcal{O} =& \sum_{n=1}^{N}\left(-\frac{1}{2\sqrt{\alpha_{n}}}\left(2\tilde{K}_{n,n}+\left(1+\boldsymbol{\kappa}_{n}\left(\boldsymbol{\mu}_{t_{n}}-\boldsymbol{\mu}_{y_{n}}\right)\right)^{2}+\boldsymbol{\kappa}_{n}\Sigma_{t_n}\boldsymbol{\kappa}_{n}^{\intercal}+\boldsymbol{\kappa}_{n}\Sigma_{y_n}\boldsymbol{\kappa}_{n}^{\intercal}-\alpha_{n}\right)\right.\nonumber\\
 & \phantom{\sum_{i=1}^{n}\left(\right)}\left.-\boldsymbol{\kappa}_{n}\left(\boldsymbol{\mu}_{t_{n}}-\boldsymbol{\mu}_{y_{n}}\right)-\frac{1}{4}\log\alpha_{n}-\log\left(B_{\frac{1}{2}}\left(\sqrt{\alpha_{n}}\right)\right)\right)\nonumber\\
 & \phantom{\sum_{i=1}^{n}\left(\right)}-\frac{1}{2}\sum_{j\in Y}\left(-\log\left|\Sigma_{j}\right|+\text{trace}\left(K^{-1}_{PP}\Sigma_{j}\right)+\boldsymbol{\mu}_{j}^{\intercal}K^{-1}_{PP}\boldsymbol{\mu}_{j}\right)\enspace.
 \end{align}

We assume that $\mathbf{f}$ and $\boldsymbol{\lambda}$ are independent $q\left(\mathbf{f},\boldsymbol{\lambda}\right)=\prod_{j\in Y}q\left(\mathbf{f}_j\right)\prod_{n=1}^{N}q\left(\lambda_{n}\right)$.
Furthermore, we impose a variational sparse approximation where, $q(\mathbf{f}_j) = \int p(\mathbf{f}_j|\mathbf{u}_j)q(\mathbf{u}_j)\mathrm{d}\mathbf{u}_j$ with $p(\mathbf{f}_j|\mathbf{u}_j) = \mathcal{N}\left(\kappa \mathbf{u}_j,\tilde{K}\right)$, $q\left(\mathbf{u}_j\right)=\mathcal{N}\left(\mathbf{u}_{j}|\boldsymbol{\mu}_{j},\Sigma_{j}\right)$, and $q\left(\lambda_{n}\right)=\mathcal{GIG}\left(\frac{1}{2},1,\alpha_{n}\right)$.
We use $p\left(\mathbf{u}_{j}\right)=\mathcal{N}\left(0,K_{PP}\right)$ as a prior on the inducing points.
Here, $\kappa=K_{NP}K^{-1}_{PP}$, $\tilde{K}=K_{NN}-K_{NP}K^{-1}_{PP}K_{PN}$, and $t_{n}=\argmax_{t\in Y,t\neq y_{n}}f_{n,t}$.
In the following, $t_n$ indicates the class index of highest class prediction for a class not equal to the true class $y_{n}$.

The goal of variational inference is to maximize the evidence lower bound (ELBO),
\begin{align}
\text{ELBO} & = \mathbb{E}_{q(\mathbf{u},\boldsymbol{\lambda})}\left[\log p\left(\mathbf{y}|\mathbf{u},\boldsymbol{\lambda} \right)\right] - \operatorname{KL}\left[q\left(\mathbf{u},\boldsymbol{\lambda} \right)||p\left(\mathbf{u},\boldsymbol{\lambda} \right)\right]\\
& = \mathbb{E}_{q(\mathbf{u},\boldsymbol{\lambda})}[\log p(\mathbf{y}|\mathbf{u},\boldsymbol{\lambda)}] + \mathbb{E}_{q(\mathbf{\mathbf{u}},\boldsymbol{\lambda})}[\log p(\mathbf{u}
,\boldsymbol{\lambda})] - \mathbb{E}_{q(\mathbf{u},\boldsymbol{\lambda})}[\log q(\mathbf{u},\boldsymbol{\lambda})]\\
& = \mathbb{E}_{q(\mathbf{u},\boldsymbol{\lambda})}[\log p(\mathbf{y},\mathbf{u},\boldsymbol{\lambda})] - \mathbb{E}_{q(\mathbf{u},\boldsymbol{\lambda})}[\log q(\mathbf{u},\boldsymbol{\lambda})]\\
& = \mathbb{E}_{q(\mathbf{u},\boldsymbol{\lambda})}\left[\log p\left(\mathbf{y},\boldsymbol{\lambda}|\mathbf{u}\right)\right]  + \mathbb{E}_{q(\mathbf{u},\boldsymbol{\lambda})}[\log p\left(\mathbf{u}\right)] - \mathbb{E}_{q(\mathbf{u},\boldsymbol{\lambda})}[\log q(\mathbf{u},\boldsymbol{\lambda})]\\
& = \mathbb{E}_{q(\mathbf{u},\boldsymbol{\lambda})}\left[\log\left[\mathbb{E}_{p(\mathbf{f}|\mathbf{u})}\left[p\left(\mathbf{y},\boldsymbol{\lambda}|\mathbf{f}\right)\right]\right]\right]  + \mathbb{E}_{q(\mathbf{u},\boldsymbol{\lambda})}[\log p\left(\mathbf{u}\right)] - \mathbb{E}_{q(\mathbf{u},\boldsymbol{\lambda})}[\log q(\mathbf{u},\boldsymbol{\lambda})]
\end{align}
We apply Jensen's inequality to the first term to obtain a tractable lower bound, 
\begin{align}
& \log\mathbb{E}_{p\left(\mathbf{f}|\mathbf{u}\right)}\left[p\left(\mathbf{y},\boldsymbol{\lambda}|\mathbf{f}\right)\right]\\
\geq & \mathbb{E}_{p\left(\mathbf{f}|\mathbf{u}\right)}\left[\log p\left(\mathbf{y},\boldsymbol{\lambda}|\mathbf{f}\right)\right]\\
= & \sum_{n=1}^{N}\mathbb{E}_{p\left(\mathbf{f}|\mathbf{u}\right)}\left[\log p\left(y_{n},\lambda_{n}|\mathbf{f}\right)\right]\\
= & \sum_{n=1}^{N}\mathbb{E}_{p\left(\mathbf{f}|\mathbf{u}\right)}\left[\log\left(\frac{1}{\sqrt{2\pi\lambda_{n}}}\exp\left(-\frac{1}{2}\frac{\left(1+\lambda_{n}+f_{t_{n}}\left(\mathbf{x}_n\right)-f_{y_{n}}\left(\mathbf{x}_n\right)\right)^{2}}{\lambda_{n}}\right)\right)\right]\\
\propto & -\frac{1}{2}\sum_{n=1}^{N}\mathbb{E}_{p\left(\mathbf{f}|\mathbf{u}\right)}\left[\log\lambda_{n}+\frac{\left(1+\lambda_{n}+f_{t_{n}}\left(\mathbf{x}_n\right)-f_{y_{n}}\left(\mathbf{x}_n\right)\right)^{2}}{\lambda_{n}}\right]\\
= & -\frac{1}{2}\sum_{n=1}^{N}\log\lambda_{n}+\frac{1}{\lambda_{n}}\mathbb{E}_{p\left(\mathbf{f}|\mathbf{u}\right)}\left[\left(1+\lambda_{n}+f_{t_{n}}\left(\mathbf{x}_n\right)-f_{y_{n}}\left(\mathbf{x}_n\right)\right)^{2}\right]\\
= & -\frac{1}{2}\sum_{n=1}^{N}\log\lambda_{n}+\frac{1}{\lambda_{n}}\left(2\tilde{K}_{n,n}+\left(1+\lambda_{n}+\boldsymbol{\kappa}_{n}\left(\mathbf{u}_{t_{n}}-\mathbf{u}_{y_{n}}\right)\right)^{2}\right)\enspace.
\end{align}
Now we simplify the expectation with respect to the approximate posterior $q(\mathbf{u},\boldsymbol{\lambda})$ as,
\begin{align}
 & \mathbb{E}_{q}\left[\log\lambda_{n}+\frac{1}{\lambda_{n}}\left(2\tilde{K}_{n,n}+\left(1+\lambda_{n}+\boldsymbol{\kappa}_{n}\left(\mathbf{u}_{t_{n}}-\mathbf{u}_{y_{n}}\right)\right)^{2}\right)\right]\\
= & \mathbb{E}_{q}\left[\log\lambda_{n}\right]+\mathbb{E}_{q}\left[\frac{1}{\lambda_{n}}\left(2\tilde{K}_{n,n}+\left(1+\lambda_{n}+\boldsymbol{\kappa}_{n}\left(\mathbf{u}_{t_{n}}-\mathbf{u}_{y_{n}}\right)\right)^{2}\right)\right]\\
= & \mathbb{E}_{q}\left[\log\lambda_{n}\right]+\mathbb{E}_{q}\left[\frac{1}{\lambda_{n}}\left(2\tilde{K}_{n,n}+1+2\lambda_{n}+2\boldsymbol{\kappa}_{n}\left(\mathbf{u}_{t_{n}}-\mathbf{u}_{y_{n}}\right)+\lambda_{n}^{2}+2\lambda_{n}\boldsymbol{\kappa}_{n}\left(\mathbf{u}_{t_{n}}-\mathbf{u}_{y_{n}}\right)\right.\right.\nonumber\\
  & \phantom{\mathbb{E}_{q}\left[\log\lambda_{n}\right]+\mathbb{E}_{q}\left[\frac{1}{\lambda_{n}}\left(\right)\right]}\left.\left. +\left(\boldsymbol{\kappa}_{n}\mathbf{u}_{t_{n}}\right)^{2}+\left(\boldsymbol{\kappa}_{n}\mathbf{u}_{y_{n}}\right)^{2}-2\boldsymbol{\kappa}_{n}\mathbf{u}_{t_{n}}\boldsymbol{\kappa}_{n}\mathbf{u}_{y_{n}}\right)\right]\\
\propto & \mathbb{E}_{q\left(\lambda_{n}\right)}\left[\log\lambda_{n}\right]+\frac{1}{\sqrt{\alpha_{n}}}\left(2\tilde{K}_{n,n}+\left(1+\boldsymbol{\kappa}_{n}\left(\boldsymbol{\mu}_{t_{n}}-\boldsymbol{\mu}_{y_{n}}\right)\right)^{2}+\boldsymbol{\kappa}_{n}\Sigma_{t_n}\boldsymbol{\kappa}_{n}^{\intercal}+\boldsymbol{\kappa}_{n}\Sigma_{y_n}\boldsymbol{\kappa}_{n}^{\intercal}\right)\nonumber\\
 & +\mathbb{E}_{q\left(\lambda_{n}\right)}\left[\lambda_{n}\right]+2\boldsymbol{\kappa}_{n}\left(\boldsymbol{\mu}_{t_{n}}-\boldsymbol{\mu}_{y_{n}}\right)\enspace.
\end{align}
\citet{Wenzel2017} derive the entropy of $q\left(\lambda_{n}\right)$ as
\begin{equation}
\mathbb{E}_{q\left(\lambda_{n}\right)}\left[\log q\left(\lambda_{n}\right)\right]\propto-\frac{1}{4}\log\left(\alpha_{n}\right)-\frac{1}{2}\mathbb{E}_{q\left(\lambda_{n}\right)}\left[\log\left(\lambda_{n}\right)\right]-\log\left(B_{\frac{1}{2}}\left(\sqrt{\alpha_{n}}\right)\right)-\frac{1}{2}\mathbb{E}_{q\left(\lambda_{n}\right)}\left[\lambda_{n}\right]-\frac{\sqrt{\alpha_{n}}}{2}\enspace,
\end{equation}
where $B_{\frac{1}{2}}\left(\cdot\right)$ is the modified Bessel function~\cite{Joergensen1982}.
Plugging these expressions into the derived lower bound on evidence lower bound leads to our variational training objective,
\begin{align}
 & \sum_{n=1}^{N}\left(-\frac{1}{2}\left(\mathbb{E}_{q\left(\lambda_{n}\right)}\left[\log\lambda_{n}\right]+\frac{1}{\sqrt{\alpha_{n}}}\left(2\tilde{K}_{n,n}+\left(1+\boldsymbol{\kappa}_{n}\left(\boldsymbol{\mu}_{t_{n}}-\boldsymbol{\mu}_{y_{n}}\right)\right)^{2}+\boldsymbol{\kappa}_{n}\Sigma_{t_n}\boldsymbol{\kappa}_{n}^{\intercal}+\boldsymbol{\kappa}_{n}\Sigma_{y_n}\boldsymbol{\kappa}_{n}^{\intercal}\right)\right.\right.\nonumber\\
 & \phantom{\sum_{i=1}^{n}\left(\right)}\left.+\mathbb{E}_{q\left(\lambda_{n}\right)}\left[\lambda_{n}\right]+2\boldsymbol{\kappa}_{n}\left(\boldsymbol{\mu}_{t_{n}}-\boldsymbol{\mu}_{y_{n}}\right)-\frac{1}{4}\log\alpha_{n}-\frac{1}{2}\mathbb{E}_{q\left(\lambda_{n}\right)}\left[\log\lambda_{n}\right]\right)\nonumber\\
 & \phantom{\sum_{i=1}^{n}\left(\right)}\left.-\log\left(B_{\frac{1}{2}}\left(\sqrt{\alpha_{n}}\right)\right)-\frac{1}{2}\mathbb{E}_{q\left(\lambda_{n}\right)}\left[\lambda_{n}\right]-\frac{\sqrt{\alpha_{n}}}{2}\right)-\mathrm{KL}\left(q\left(\mathbf{u}\right)\bigparallel p\left(\mathbf{u}\right)\right)\\
& = \sum_{n=1}^{N}\left(-\frac{1}{2\sqrt{\alpha_{n}}}\left(2\tilde{K}_{n,n}+\left(1+\boldsymbol{\kappa}_{n}\left(\boldsymbol{\mu}_{t_{n}}-\boldsymbol{\mu}_{y_{n}}\right)\right)^{2}+\boldsymbol{\kappa}_{n}\Sigma_{t_n}\boldsymbol{\kappa}_{n}^{\intercal}+\boldsymbol{\kappa}_{n}\Sigma_{y_n}\boldsymbol{\kappa}_{n}^{\intercal}-\alpha_{n}\right)\right.\nonumber\\
 & \phantom{\sum_{i=1}^{n}\left(\right)}\left.-\boldsymbol{\kappa}_{n}\left(\boldsymbol{\mu}_{t_{n}}-\boldsymbol{\mu}_{y_{n}}\right)-\frac{1}{4}\log\alpha_{n}-\log\left(B_{\frac{1}{2}}\left(\sqrt{\alpha_{n}}\right)\right)\right)\nonumber\\
 & \phantom{\sum_{i=1}^{n}\left(\right)}-\frac{1}{2}\sum_{j\in Y}\left(-\log\left|\Sigma_{j}\right|+\text{trace}\left(K^{-1}_{PP}\Sigma_{j}\right)+\boldsymbol{\mu}_{j}^{\intercal}K^{-1}_{PP}\boldsymbol{\mu}_{j}\right) = \mathcal{O}
\end{align}

\section{Derivation of Euclidean Gradients}

We derive the Euclidean gradients with respect to $\boldsymbol{\mu}_j$, $\Sigma_j$ and $\alpha_n$ using the notation from the previous section.

\begin{align}
\frac{\partial\mathcal{O}_{n}}{\partial\boldsymbol{\mu}_{j}} & =\frac{\partial}{\partial\boldsymbol{\mu}_{j}}\left[-\frac{1}{2\sqrt{\alpha_{n}}}\left(1+\boldsymbol{\kappa}_{n}\left(\boldsymbol{\mu}_{t_{n}}-\boldsymbol{\mu}_{y_{n}}\right)\right)^{2}-\boldsymbol{\kappa}_{n}\left(\boldsymbol{\mu}_{t_{n}}-\boldsymbol{\mu}_{y_{n}}\right)-\frac{1}{2N}\sum_{j\in Y}\boldsymbol{\mu}_{j}^{\intercal}K_{PP}^{-1}\boldsymbol{\mu}_{j}\right]\\
 & =\begin{cases}
-\frac{1}{\sqrt{\alpha_{n}}}\left(\boldsymbol{\kappa}_{n}^{T}+\boldsymbol{\kappa}_{n}^{T}\boldsymbol{\kappa}_{n}\boldsymbol{\mu}_{j}-\boldsymbol{\kappa}_{n}^{T}\boldsymbol{\kappa}_{n}\boldsymbol{\mu}_{y_{n}}\right)-\boldsymbol{\kappa}_{n}^{T}-\frac{1}{N}K_{PP}^{-1}\boldsymbol{\mu}_{j} & j=t_{n}\\
\frac{1}{\sqrt{\alpha_{n}}}\left(\boldsymbol{\kappa}_{n}^{T}+\boldsymbol{\kappa}_{n}^{T}\boldsymbol{\kappa}_{n}\boldsymbol{\mu}_{t_{n}}-\boldsymbol{\kappa}_{n}^{T}\boldsymbol{\kappa}_{n}\boldsymbol{\mu}_{j}\right)+\boldsymbol{\kappa}_{n}^{T}-\frac{1}{N}K_{PP}^{-1}\boldsymbol{\mu}_{j} & j=y_{n}\\
-\frac{1}{N}K_{PP}^{-1}\boldsymbol{\mu}_{j} & \text{otherwise}
\end{cases}
\end{align}
\begin{align}
\frac{\partial\mathcal{O}_{n}}{\partial\Sigma_{j}} & =\frac{\partial}{\partial\Sigma_{j}}\left[-\frac{1}{2\sqrt{\alpha_{n}}}\left(\boldsymbol{\kappa}_{n}\Sigma_{y_{n}}\boldsymbol{\kappa}_{n}^{\intercal}+\boldsymbol{\kappa}_{n}\Sigma_{t_{n}}\boldsymbol{\kappa}_{n}^{\intercal}\right)-\frac{1}{2N}\sum_{j\in Y}\left(-\log\left|\Sigma_{j}\right|+\text{trace}\left(K_{PP}^{-1}\Sigma_{j}\right)\right)\right]\\
 & =\begin{cases}
-\frac{1}{2\sqrt{\alpha_{n}}}\boldsymbol{\kappa}_{n}^{T}\boldsymbol{\kappa}_{n}-\frac{1}{2N}\left(-\Sigma_{j}^{-1}+K_{PP}^{-1}\right) & j=t_{n}\lor j=y_{n}\\
-\frac{1}{2N}\left(-\Sigma_{j}^{-1}+K_{PP}^{-1}\right) & \text{otherwise}
\end{cases}
\end{align}
\begin{align}
\frac{\partial\mathcal{O}}{\partial\alpha_{n}} & =\frac{\partial}{\partial\alpha_{n}}\left[-\frac{1}{2\sqrt{\alpha_{n}}}\left(2\tilde{K}_{n,n}+\left(1+\boldsymbol{\kappa}_{n}\left(\boldsymbol{\mu}_{t_{n}}-\boldsymbol{\mu}_{y_{n}}\right)\right)^{2}+\boldsymbol{\kappa}_{n}\Sigma_{y_{n}}\boldsymbol{\kappa}_{n}^{\intercal}+\boldsymbol{\kappa}_{n}\Sigma_{t_{n}}\boldsymbol{\kappa}_{n}^{\intercal}-\alpha_{n}\right)\right.\nonumber\\
 & \phantom{\frac{\partial}{\partial\alpha_{n}}\left[\frac{1}{1}\right]}\left.-\frac{1}{4}\log\alpha_{n}-\log\left(B_{\frac{1}{2}}\left(\sqrt{\alpha_{n}}\right)\right)\right]\\
 & =\frac{1}{4\sqrt{\alpha_{n}}^{3}}\left(2\tilde{K}_{n,n}+\left(1+\boldsymbol{\kappa}_{n}\left(\boldsymbol{\mu}_{t_{n}}-\boldsymbol{\mu}_{y_{n}}\right)\right)^{2}+\boldsymbol{\kappa}_{n}\Sigma_{y_{n}}\boldsymbol{\kappa}_{n}^{\intercal}+\boldsymbol{\kappa}_{n}\Sigma_{t_{n}}\boldsymbol{\kappa}_{n}^{\intercal}\right)+\frac{1}{4\sqrt{\alpha_{n}}}\nonumber\\
 & \phantom{=}+\frac{1}{4\alpha_{n}}-\left(\frac{1}{4\alpha_{n}}+\frac{1}{2\sqrt{\alpha_{n}}}\right)\\
 & =\frac{1}{4\sqrt{\alpha_{n}}^{3}}\left(2\tilde{K}_{n,n}+\left(1+\boldsymbol{\kappa}_{n}\left(\boldsymbol{\mu}_{t_{n}}-\boldsymbol{\mu}_{y_{n}}\right)\right)^{2}+\boldsymbol{\kappa}_{n}\Sigma_{y_{n}}\boldsymbol{\kappa}_{n}^{\intercal}+\boldsymbol{\kappa}_{n}\Sigma_{t_{n}}\boldsymbol{\kappa}_{n}^{\intercal}\right)-\frac{1}{4\sqrt{\alpha_{n}}}
\end{align}

\section{Inference with Coordinate Ascent}
\begin{algorithm}[t]
\caption{Coordinate Ascent for inferring variational parameters of Multi-Class Bayesian SVM}
\begin{algorithmic}[1]
\label{alg:ca4svm}
\REQUIRE{Data set $D=\{x_n,y_n\}_{n=1}^N$, learning rate $\rho_t$, initial inducing points}
\STATE Compute $\tilde{K}$, $K_{NP}$ and $K_{PP}$.
\FOR{$t = 1,\ldots,T$}
  \FOR{$n = 1,\ldots,N$}
    \STATE $\alpha_n \leftarrow 2\tilde{K}_{n,n}+\left(1+\boldsymbol{\kappa}_{n}\left(\boldsymbol{\mu}_{t_{n}}-\boldsymbol{\mu}_{y_{n}}\right)\right)^{2}+\boldsymbol{\kappa}_{n}\Sigma_{y_{n}}\boldsymbol{\kappa}_{n}^{\intercal}+\boldsymbol{\kappa}_{n}\Sigma_{t_{n}}\boldsymbol{\kappa}_{n}^{\intercal}$.
  \ENDFOR
  \FOR{$j\in Y$}
    \STATE Compute $\hat{\boldsymbol{\eta}}_{1,j}$ according to Equation \eqref{eq:update-eta1}.
    \STATE Compute $\hat{\boldsymbol{\eta}}_{2,j}$ according to Equation \eqref{eq:update-eta2}.
    \STATE $\boldsymbol{\eta}_{1,j}\leftarrow\left(1-\rho_t\right)\boldsymbol{\eta}_{1,j}+\rho_t \hat{\boldsymbol{\eta}}_{1,j}$
    \STATE $\boldsymbol{\eta}_{2,j}\leftarrow\left(1-\rho_t\right)\boldsymbol{\eta}_{2,j}+\rho_t \hat{\boldsymbol{\eta}}_{2,j}$
    \STATE $\Sigma_j \leftarrow -\frac{1}{2}\boldsymbol{\eta}_{2,j}^{-1}$
    \STATE $\mu_j \leftarrow \Sigma_j\boldsymbol{\eta}_{1,j}$
  \ENDFOR
\ENDFOR
\RETURN $\alpha$, $\mu$, $\Sigma$
\end{algorithmic}
\end{algorithm}

\citet{Wenzel2017} propose the use of coordinate ascent in order to speed up the inference.
While we relied in our experiments on Euclidean gradients combined with stochastic gradient ascent based algorithms, this inference scheme is also applicable for our proposed multi-class inference.
For completeness we derive this algorithm as well.

First, we require to derive the natural gradients by multiplying the Euclidean gradients with the inverse Fisher information matrix~\cite{Amari2007}.
Since we apply it to a Gaussian distribution, this results in the following natural gradients
\[
\widetilde{\nabla}_{\boldsymbol{\eta}_{1,j},\boldsymbol{\eta}_{2,j}}\mathcal{O}_{n}=\left(\frac{\partial\mathcal{O}_{n}}{\partial\boldsymbol{\mu}_{j}}-2\frac{\partial\mathcal{O}_{n}}{\partial\Sigma_{j}}\boldsymbol{\mu}_{j},\ \frac{\partial\mathcal{O}_{n}}{\partial\Sigma_{j}}\right)
\]
Using the identities $\boldsymbol{\eta}_{1,j}=\Sigma_{j}^{-1}\boldsymbol{\mu}_{j}$ and $\boldsymbol{\eta}_{2,j}=-\frac{1}{2}\Sigma_{j}^{-1}$, we derive the final natural gradients,
\begin{align}
\widetilde{\nabla}_{\boldsymbol{\eta}_{2,j}}\mathcal{O}_{n} & =-\frac{1}{2N}\left(-2\boldsymbol{\eta}_{2,j}+K_{PP}^{-1}\right)-\begin{cases}
\frac{1}{2\sqrt{\alpha_{n}}}\boldsymbol{\kappa}_{n}^{T}\boldsymbol{\kappa}_{n} & j=t_{n}\lor j=y_{n}\\
0 & \text{otherwise}
\end{cases}\\
\widetilde{\nabla}_{\boldsymbol{\eta}_{1,j}}\mathcal{O}_{n} & =\left(\frac{1}{\sqrt{\alpha_{n}}}\left(\boldsymbol{\kappa}_{n}^{T}+\frac{1}{2}\boldsymbol{\kappa}_{n}^{T}\boldsymbol{\kappa}_{n}\boldsymbol{\eta}_{2,t_{n}}^{-1}\boldsymbol{\eta}_{1,t_{n}}-\frac{1}{2}\boldsymbol{\kappa}_{n}^{T}\boldsymbol{\kappa}_{n}\boldsymbol{\eta}_{2,y_{n}}^{-1}\boldsymbol{\eta}_{1,y_{n}}\right)+\boldsymbol{\kappa}_{n}^{T}\right)\cdot\begin{cases}
-1 & j=t_{n}\\
1 & j=y_{n}\\
0 & \text{otherwise}
\end{cases}\nonumber\\
 & -\frac{1}{2N}K_{PP}^{-1}\boldsymbol{\eta}_{2,j}^{-1}\boldsymbol{\eta}_{1,j}\nonumber\\
 & -2\left(-\frac{1}{2N}\left(-2\boldsymbol{\eta}_{2,j}+K_{PP}^{-1}\right)-\begin{cases}
\frac{1}{2\sqrt{\alpha_{n}}}\boldsymbol{\kappa}_{n}^{T}\boldsymbol{\kappa}_{n} & j=t_{n}\lor j=y_{n}\\
0 & \text{otherwise}
\end{cases}\right)\frac{1}{2}\boldsymbol{\eta}_{2,j}^{-1}\boldsymbol{\eta}_{1,j}\\
 & =-\frac{1}{N}\boldsymbol{\eta}_{1,j}+\begin{cases}
-\left(\frac{1}{\sqrt{\alpha_{n}}}\left(\boldsymbol{\kappa}_{n}^{T}-\frac{1}{2}\boldsymbol{\kappa}_{n}^{T}\boldsymbol{\kappa}_{n}\boldsymbol{\eta}_{2,y_{n}}^{-1}\boldsymbol{\eta}_{1,y_{n}}\right)+\boldsymbol{\kappa}_{n}^{T}\right) & j=t_{n}\\
\left(\frac{1}{\sqrt{\alpha_{n}}}\left(\boldsymbol{\kappa}_{n}^{T}+\frac{1}{2}\boldsymbol{\kappa}_{n}^{T}\boldsymbol{\kappa}_{n}\boldsymbol{\eta}_{2,t_{n}}^{-1}\boldsymbol{\eta}_{1,t_{n}}\right)+\boldsymbol{\kappa}_{n}^{T}\right) & j=y_{n}\\
0 & \text{otherwise}
\end{cases}\\
 & =-\frac{1}{N}\boldsymbol{\eta}_{1,j}+\boldsymbol{\kappa}_{n}^{T}\begin{cases}
-\left(\frac{1}{\sqrt{\alpha_{n}}}\left(1-\frac{1}{2}\boldsymbol{\kappa}_{n}\boldsymbol{\eta}_{2,y_{n}}^{-1}\boldsymbol{\eta}_{1,y_{n}}\right)+1\right) & j=t_{n}\\
\left(\frac{1}{\sqrt{\alpha_{n}}}\left(1+\frac{1}{2}\boldsymbol{\kappa}_{n}\boldsymbol{\eta}_{2,t_{n}}^{-1}\boldsymbol{\eta}_{1,t_{n}}\right)+1\right) & j=y_{n}\\
0 & \text{otherwise}
\end{cases}
\end{align}
Equating the gradients to zero yields the following update rules
\begin{align}
\boldsymbol{\eta}_{1,j} & =-N\boldsymbol{\kappa}_{n}^{T}\begin{cases}
-\left(\frac{1}{\sqrt{\alpha_{n}}}\left(1-\frac{1}{2}\boldsymbol{\kappa}_{n}\boldsymbol{\eta}_{2,y_{n}}^{-1}\boldsymbol{\eta}_{1,y_{n}}\right)+1\right) & j=t_{n}\\
\left(\frac{1}{\sqrt{\alpha_{n}}}\left(1+\frac{1}{2}\boldsymbol{\kappa}_{n}\boldsymbol{\eta}_{2,t_{n}}^{-1}\boldsymbol{\eta}_{1,t_{n}}\right)+1\right) & j=y_{n}\\
0 & \text{otherwise}
\end{cases}\label{eq:update-eta1}\\
\boldsymbol{\eta}_{2,j} & =-\frac{1}{2}K_{PP}^{-1}-N\begin{cases}
\frac{1}{2\sqrt{\alpha_{n}}}\boldsymbol{\kappa}_{n}^{T}\boldsymbol{\kappa}_{n} & j=t_{n}\lor j=y_{n}\\
0 & \text{otherwise}
\end{cases}\label{eq:update-eta2}
\end{align}

Finally, we summarize the inference algorithm in Algorithm \ref{alg:ca4svm}.
For training the remaining parameters such as inducing point locations and kernel hyperparameters, we propose to use standard gradient updates after every few coordinate ascent steps.
This learning of hyperparameters is often referred to as Type II maximum likelihood~\cite{Rasmussen2006}.

\end{document}